\documentclass{article}

\usepackage[dvipsnames]{xcolor}         
\usepackage{aligned-overset}
\usepackage{tikz}
\usepackage{makecell}
\usepackage{amsthm}
\usepackage{stfloats}
\usepackage{bbm}
\usepackage{caption}
\usetikzlibrary{positioning, bayesnet}
\usepackage{multirow}
\usepackage{tablefootnote}
\usepackage{textcomp}
\usepackage{amsmath}

\usepackage[utf8]{inputenc} 
\usepackage[T1]{fontenc}    
\usepackage{hyperref}       
\hypersetup{colorlinks=false}
\usepackage{url}            
\usepackage{booktabs}       
\usepackage{amsfonts}       
\usepackage{nicefrac}       
\usepackage{microtype}      
\usepackage{bm}
\usepackage{graphicx}
\usepackage{subfigure}

\usepackage[accepted]{icml2024}

\usepackage{algorithm, algorithmic}

\usepackage{amsthm,amsmath,amssymb, amscd}
\usepackage{mathtools}

\usepackage{amsmath,amsfonts,bm}

\def\eqref#1{equation~\ref{#1}}

\def\1{\bm{1}}

\DeclareMathAlphabet{\mathsfit}{\encodingdefault}{\sfdefault}{m}{sl}
\SetMathAlphabet{\mathsfit}{bold}{\encodingdefault}{\sfdefault}{bx}{n}

\newcommand{\E}{\mathbb{E}}

\newcommand{\R}{\mathbb{R}}

\newcommand{\Var}{\mathrm{Var}}

\newcommand{\dif}{\mathrm{d}}

\usepackage{pifont}
\usepackage[capitalize,noabbrev]{cleveref}

\theoremstyle{plain}
\newtheorem{theorem}{Theorem}[section]
\newtheorem{proposition}[theorem]{Proposition}

\theoremstyle{definition}

\newtheorem{assumption}[theorem]{Assumption}
\theoremstyle{remark}

\newcommand{\cirone}{\text{\ding{172}}}
\newcommand{\cirtwo}{\text{\ding{173}}}
\newcommand{\cirthree}{\text{\ding{174}}}

\newcommand{\diag}{\textit{diag}}
\newcommand{\poly}{\textit{poly}}

\renewcommand{\eqref}[1]{(\ref{#1})}

\usepackage[textsize=tiny]{todonotes}

\icmltitlerunning{Kernel Semi-Implicit Variational Inference}

\begin{document}

\twocolumn[
\icmltitle{Kernel Semi-Implicit Variational Inference}



\icmlsetsymbol{equal}{*}

\begin{icmlauthorlist}
\icmlauthor{Ziheng Cheng$^\ast$}{sms}
\icmlauthor{Longlin Yu$^\ast$}{sms}
\icmlauthor{Tianyu Xie}{sms}
\icmlauthor{Shiyue Zhang}{sms}
\icmlauthor{Cheng Zhang}{sms,css}
\end{icmlauthorlist}

\icmlaffiliation{sms}{School of Mathematical Sciences, Peking University, China}
\icmlaffiliation{css}{Center for Statistical Science, Peking University, China}

\icmlcorrespondingauthor{Cheng Zhang}{chengzhang@math.pku.edu.cn}

\icmlkeywords{Machine Learning, ICML}

\vskip 0.3in
]



\printAffiliationsAndNotice{\icmlEqualContribution} 

\begin{abstract}
Semi-implicit variational inference (SIVI) extends traditional variational families with semi-implicit distributions defined in a hierarchical manner.
Due to the intractable densities of semi-implicit distributions, classical SIVI often resorts to surrogates of evidence lower bound (ELBO) that would introduce biases for training.
A recent advancement in SIVI, named SIVI-SM, utilizes an alternative score matching objective made tractable via a minimax formulation, albeit requiring an additional lower-level optimization.
In this paper, we propose kernel SIVI (KSIVI), a variant of SIVI-SM that eliminates the need for lower-level optimization through kernel tricks.
Specifically, we show that when optimizing over a reproducing kernel Hilbert space (RKHS), the lower-level problem has an explicit solution. 
This way, the upper-level objective becomes the kernel Stein discrepancy (KSD), which is readily computable for stochastic gradient descent due to the hierarchical structure of semi-implicit variational distributions.
An upper bound for the variance of the Monte Carlo gradient estimators of the KSD objective is derived, which allows us to establish novel convergence guarantees of KSIVI.
We demonstrate the effectiveness and efficiency of KSIVI on both synthetic distributions and a variety of real data Bayesian inference tasks.
\end{abstract}

\section{Introduction}
Variational inference (VI) is an optimization based approach widely used to approximate posterior densities for Bayesian models \citep{Jordan1999AnIT, Wainwright08, Blei2016Variational}.
The idea behind VI is to first posit a family of variational distributions over the model parameters (or latent variables), and then to find a member from this family that is close to the target posterior, where the closeness is often measured by the Kullback-Leibler (KL) divergence.
Since the posterior generally does not possess analytical densities, in practice people often maximize the evidence lower bound (ELBO) instead, which is an equivalent but tractable reformulation \citep{Jordan1999AnIT}.

In classical VI, a common assumption is that variational distributions are factorized over the parameters (or latent variables) \citep{bishop2000variational, Hinton1993KeepingTN, Peterson1989}.
When combined with conditional conjugate models, this mean-field assumption allows for a straightforward coordinate-ascent optimization scheme with closed-form update rules \citep{Blei2016Variational}.
However, conditional conjugacy may not hold in practice, and a factorized variational distribution may fail to capture the complex posterior due to the ignorance of the dependencies between different factorization components.
Recent years have witnessed significant advancements in the field of VI, with numerous efforts aimed at alleviating these constraints.
To address the conjugacy condition, black-box VI methods have been proposed for a broad class of models that allow generic training algorithms via Monte Carlo gradient estimators \citep{Paisley12, Nott12, ranganath2014black, Rezende14, VAE, titsias14}.
In parallel, more flexible variational families have emerged, either explicitly incorporating certain structures among the parameters or drawing inspiration from invertible transformations of probability distributions \citep{Jaakkola98, Saul96, Giordano15, Tran15, NF, RealNVP, IAF, Papamakarios19}.
Despite their successes, it is noteworthy that all these approaches assume tractable densities of variational distributions.

To further enhance the capacity of variational families, one strategy is to incorporate the implicit models defined through neural network mappings \citep{Huszar17, Tran17, AVB, Shi18, shi2018, SSM}.
These implicit models can provide more flexible variational approximations that allow easy sampling procedures by design.
However, their probability densities are often intractable, rendering direct computation of the log density ratio, crucial for evaluating the ELBO, impossible.
Implicit VI methods, therefore, generally rely on density ratio estimation for training, which not only introduces additional optimization complexity
but also is known to be difficult in high dimensional settings \citep{sugiyama2012density}.

To circumvent the need for density ratio estimation, semi-implicit variational inference (SIVI) has been introduced \citep{yin2018semi, moens2021}.
In SIVI, variational distributions are crafted using a semi-implicit hierarchical framework which allows efficient computation of surrogate ELBOs for training.
As surrogate ELBOs are inherently biased, \citet{titsias2019unbiased} proposed an unbiased estimator of the gradient of the exact ELBO, albeit requiring an inner-loop MCMC sampling from a reversed conditional distribution that can easily become expensive in high-dimensional regimes.
Alternatively, \citet{yu2023semi} introduced SIVI-SM, a variant of SIVI that uses an objective based on Fisher divergence rather than KL divergence. 
They showed that this new objective facilitates a simple minimax formulation, where the intractability of densities can be naturally handled similarly to denoising score matching (DSM) \citep{Vincent2011}, due to the hierarchical construction of semi-implicit distributions.
However, this comes at the price of an additional lower-level problem that requires optimization over a family of functions parameterized via neural networks.

As a popular approach in machine learning, kernel methods have shown great advantage when it comes to optimization over function spaces \citep{MMD, liu2016stein}.
In this paper, we propose kernel semi-implicit variational inference (KSIVI), a variant of SIVI-SM that eliminates the need for lower-level optimization through kernel tricks.
More specifically, we show that there exists an explicit solution to the lower-level problem when optimizing over a reproducing kernel Hilbert space (RKHS).
Upon substitution, the objective for the upper-level problem becomes the kernel Stein discrepancy (KSD), a widely used kernel-induced dissimilarity measure between probability distributions \citep{liu2016kernelized, pmlr-v48-chwialkowski16, gorham2017measuring}.
Moreover, similarly as done in DSM, this solution has a tractable form that solely involves conditional densities due to the hierarchical structure of semi-implicit variational families, making KSD readily computable for stochastic gradient descent (SGD).
We derive an upper bound for the variance of the Monte Carlo gradient estimators under mild assumptions.
This bound allows us to establish convergence guarantees of KSIVI to a stationary point using standard analysis in stochastic optimization.
We show empirically that KSIVI performs on par or better than SIVI-SM on both synthetic distributions and a variety of real data Bayesian inference tasks while enjoying more efficient computation, more stable training, and fewer hyperparameter tuning headaches.

\section{Background}

\paragraph{Notations}
Let $\|\cdot\|$ be the standard Euclidean norm of a vector and the operator norm of a matrix or high-dimensional tensor. 
Let $\|\cdot\|_{\infty}$ denote the $\ell_\infty$ norm of a vector.
For any $x,y\in\R^d$, the expressions $x\odot y, \frac{x}{y}, x^2$ stand for element-wise product, division, and square, respectively.
The RKHS induced by kernel $k(\cdot,\cdot)$ is denoted as $\mathcal{H}_0$ and the corresponding Cartesian product is defined as $\mathcal{H}:=\mathcal{H}_0^{\otimes d}$.
Let $\nabla_1k (\text{resp.} \nabla_2k)$ be the gradient of the kernel w.r.t. the first (resp. second) variable.
We denote the inner product in $\R^d$ and Hilbert space $\mathcal{F}$ by $\left\langle \cdot,\cdot \right\rangle$ and $\left\langle \cdot,\cdot \right\rangle_{\mathcal{F}}$, respectively.
For a probabilistic density function $q$, we use $s_q$ to denote its score function $\nabla\log q$.
Finally, we use standard $\lesssim$ and $\gtrsim$ to omit constant factors.


\subsection{Semi-Implicit Variational Inference}
Given an observed data set $\mathcal{D}$, the classical VI methods find the best approximation in a variational family $\{q_\phi(x)\}_{\phi\in\Phi}$ to the posterior distribution $p(x|\mathcal{D})$ by maximizing the following evidence lower bound (ELBO)
\begin{equation}\label{eq:elbo}
    \mathcal{L}(q_\phi(x))= \E_{q_\phi(x)}\left[\log p(\mathcal{D},x)-\log q_\phi(x)\right],
\end{equation}
where $\phi$ are the variational parameters.
To expand the capacity of variational families, SIVI \citep{yin2018semi} constructs a semi-implicit variational family with hierarchical structure as
\begin{align}
q_\phi(x) = \int q_\phi(x|z)q(z)\dif z,
\end{align}
where $q(z)$ is the mixing distribution which is easy to sample from and can be implicit, and $q_\phi(x|z)$ is the conditional layer that is required to be explicit.
Compared to classical VI which requires tractable variational distributions with explicit densities, the semi-implicit construction greatly enhances the approximation ability, making it able to capture the complex dependence among different variables.

Since $q_\phi(x)$ in the ELBO \eqref{eq:elbo} is no longer tractable, \citet{yin2018semi} considers a sequence of surrogates $\underline{\mathcal{L}}^{(K)}(q_\phi(x))$ of the ELBO \eqref{eq:elbo} defined as
\begin{equation}
\underline{\mathcal{L}}^{(K)}(q_\phi(x)) = \E_{q(z^{(0:K)})q_\phi(x|z^{(0)})}\log \frac{p(\mathcal{D},x)}{\frac{1}{K+1}\sum_{k=0}^K q_\phi(x|z^{(k)})},
\end{equation}
where $q(z^{(0:K)})=\prod_{i=0}^Kq(z^{(k)})$.
The introduced surrogate $\underline{\mathcal{L}}^{(K)}(q_\phi(x))$ is a lower bound of the $\mathcal{L}(q_\phi(x))$ and is asymptotically exact in the sense that $\lim_{K\rightarrow  \infty}\underline{\mathcal{L}}^{(K)}(q_\phi(x)) = \mathcal{L}(q_\phi(x))$.

Besides the ELBO-based objectives, the score-based measure has also been introduced for VI \citep{liu2016kernelized, zhang18, Hu18, korba2021kernel} where the score function $s_p(x)$ of the target posterior distribution $p(x)=p(x|\mathcal{D})$ is generally assumed to be tractable. 
In particular, SIVI-SM \citep{yu2023semi} learns the semi-implicit approximation by minimizing the Fisher divergence 
via a minimax formulation
\begin{equation}\label{eq:minmax}
\min_\phi \max_{f}\ \E_{q_\phi(x)} \left[ 2f(x)^T [s_p(x) - s_{q_\phi}(x)]- \|f(x)\|^2\right].
\end{equation}
Using a similar trick in DSM~\citep{Vincent2011}, \eqref{eq:minmax} has a tractable form as follows
\begin{equation}
\min_\phi \max_{f}\ \E_{q_{\phi}(x,z)}\left[2f(x)^T[s_p(x) - s_{q_\phi(\cdot|z)}(x)] - \|f(x)\|^2\right],
\end{equation}
where $q_\phi(x,z)=q_\phi(x|z)q(z)$.
The introduced auxiliary function $f(x)$ is implemented as a neural network $f_\psi(x)$ in practice, necessitating additional attention to training deep models and tuning hyperparameters.

\subsection{Kernel Stein Discrepancy}
Consider a continuous and positive semi-definite kernel $k(\cdot,\cdot): \R^d \times \R^d \to \R$ and its corresponding RKHS $\mathcal{H}_0$ of real-valued functions in $\R^d$.
The reproducing property of $\mathcal{H}_0$ indicates that for any function $f\in \mathcal{H}_0$, $\left\langle f(\cdot), k(x,\cdot)\right\rangle_{\mathcal{H}_0}=f(x)$.
Given a measure $q$ such that $\int k(x,x)\dif q(x)<\infty$, it follows that $\mathcal{H}_0 \subset L^2(q)$, and the integral operator $S_{q,k}:L^2(q)\to \mathcal{H}_0$ is defined as 
\begin{equation}
    (S_{q,k}f)(\cdot) := \int k(\cdot,y)f(y)\dif q(y).
\end{equation}
The Cartesian product $\mathcal{H}:=\mathcal{H}_0^{\otimes d}$, which consists $f=(f_1, \cdots, f_d)$ with $f_i\in \mathcal{H}_0$, is an extension to vector-valued functions.
The inner product in $\mathcal{H}$ is defined by $\left\langle f,g\right\rangle_{\mathcal{H}}=\sum_{i=1}^d \left\langle f_i, g_i\right\rangle_{\mathcal{H}_0}$.
For $f\in \mathcal{H}$ and $g_0\in \mathcal{H}_0$, we extend $g_0$ and define $\left\langle f,g_0\right\rangle_{\mathcal{H}}= (\left\langle f_1, g_0\right\rangle_{\mathcal{H}_0},\cdots,\left\langle f_d, g_0\right\rangle_{\mathcal{H}_0})$.
For vector-valued inputs $f$, we reload the definition of the operator $S_{q,k}$ and apply it element-wisely.
The Stein discrepancy \citep{Gorham2015, liu2016stein,liu2016kernelized} which measures the difference between $q$ and $p$ is defined as 
\begin{equation}\label{eq:stein-discrepancy}
    \mathrm{S}(q\|p) := \max_{f\in\mathcal{F}}\ \E_{q}[\nabla\log p(x)^T f(x) + \nabla \cdot f(x)],
\end{equation}
where $\mathcal{F}$ is a pre-defined function space.
To tackle the minimization problem of $\mathrm{S}(q\|p)$, the kernel Stein discrepancy (KSD) \citep{liu2016stein, pmlr-v48-chwialkowski16, gorham2017measuring} maximizes $f$ within the unit ball of the RKHS, i.e., $\mathcal{F} = \{f:\|f\|_{\mathcal{H}}\le 1\}$. Since the optima of $f(x)$ has a closed form, KSD has the following equivalent form
\begin{equation}
    \begin{aligned}
        \textrm{KSD}(q\|p) 
        &= \left\|S_{q,k}\nabla\log\frac{p}{q}\right\|_{\mathcal{H}} \\
        &= \left\langle \nabla\log \frac{p}{q}, S_{q,k}\nabla\log\frac{p}{q}\right\rangle_{L^2(q)}^{1/2},
    \end{aligned}
\end{equation}
which can be also seen as a kernelized Fisher divergence \cite{korba2021kernel}.
\citet{gorham2017measuring} proves that under mild conditions of target distribution $p(x)$ and kernel $k$, weak convergence of the KSD implies the weak convergence of distribution.

\section{Proposed Method}
In this section, we present kernel semi-implicit variational inference (KSIVI), which explicitly solves the lower-level optimization in SIVI-SM through kernel tricks.
We will first give the expression of this explicit solution and then discuss the practical implementation of the upper-level optimization.
\subsection{Training Objective}

Instead of considering $f\in L^2(q_\phi)$ which in general does not allow an explicit solution of $f$ in \eqref{eq:minmax}, we seek the optimal $f^\ast$ in an RKHS $\mathcal{H}$ and reformulate this minimax problem as
\begin{equation}\label{eq:kernel_minmax}
    \min_{\phi}\max_{f\in \mathcal{H}} \ \E_{q_{\phi}(x)}\left[2f(x)^T[s_p(x) - s_{q_\phi}(x)] - \|f\|_{\mathcal{H}}^2\right].
\end{equation}
The following theorem shows that the solution $f^\ast$ to the lower-level optimization in \eqref{eq:kernel_minmax} has an explicit form, which allows us to reduce \eqref{eq:kernel_minmax} to a standard optimization problem.
\begin{theorem}\label{thm:opt_f}
    For any variational distribution $q_\phi$, the solution $f^*$ to the lower-level optimization in  \eqref{eq:kernel_minmax} takes the form
    \begin{equation}
        f^*(x)=\E_{ q_\phi(y)} k(x,y)\left[s_p(y)-s_{q_\phi}(y)\right].
    \end{equation}
    Thus the upper-level optimization problem for $\phi$ is
    \begin{equation}\label{eq:semi_ksd}
        \min_\phi\ \text{KSD}(q_\phi \|p)^2=\left\|S_{q_\phi,k}\nabla\log\frac{p}{q_\phi}\right\|_{\mathcal{H}}^2.
    \end{equation}
\end{theorem}

The detailed proof is deferred to Appendix \ref{app:subsec:proof_opt_f}.
Moreover, by taking advantage of the semi-implicit structure, we have
\begin{equation}
    \begin{aligned}
        &\E_{q_\phi(y)} k(x,y)s_{q_\phi}(y) \\
        = & \E_{q_\phi(y)} \frac{k(x,y)}{q_\phi(y)} \int \nabla_y q_\phi(y|z)q(z) \dif z \\
        = & \int k(x,y)\nabla_y q_{\phi}(y|z)q(z)\dif z\dif y \\
        = & \int k(x,y)s_{q_\phi(\cdot|z)}(y)q_\phi(y|z)q(z)\dif z\dif y \\
        = & \E_{q_\phi(y,z)} k(x,y)s_{q_\phi(\cdot|z)}(y).
    \end{aligned}
\end{equation}
This leads to the following proposition.
\begin{proposition}
    The solution $f^*$ in Theorem \ref{thm:opt_f} can be rewritten as
    \begin{equation}
        f^*(x)=\E_{q_\phi(y,z)} k(x,y)\left[s_p(y)-s_{q_\phi(\cdot|z)}(y)\right].
    \end{equation}
    And the $\text{KSD}(q_\phi \|p)^2$ in \eqref{eq:semi_ksd} has an equivalent expression

\begin{equation}\label{eq:phi_obj}
\begin{split}
\text{KSD}(q_\phi \|p)^2&=\E_{q_\phi(x,z), q_\phi(x
',z')} \big[k(x,x')\cdot \\
\big\langle s_p(x)&-s_{q_\phi(\cdot|z)}(x),s_p(x')-s_{q_\phi(\cdot|z')}(x')\big\rangle\big].
\end{split}
\end{equation}
\end{proposition}

Since the semi-implicit variational distribution $q_\phi(x,z)$ enables efficient sampling, both $f^*$ and $\text{KSD}(q_\phi \|p)^2$ can be estimated using the Monte Carlo method with samples from $q_\phi(x,z)$.
This way, we have transformed the minimax problem \eqref{eq:kernel_minmax} into a standard optimization problem with a tractable objective function.

\subsection{Practical Implementation}\label{sec:practical-implementation}
Suppose the conditional $q_\phi(x|z)$ admits the reparameterization trick \citep{VAE,titsias14,Rezende14}, i.e., $x\sim q_\phi(\cdot|z)\Leftrightarrow x=T_\phi(z,\xi), \xi\sim q_\xi$, where $T_\phi$ is a parameterized transformation and $q_\xi$ is a base distribution that does not depend on $\phi$.
We now show how to find an optimal variational approximation $q_{\phi}(x)$ that minimizes the KSD between $q_{\phi}(x)$ and $p(x)$ defined in \eqref{eq:phi_obj} using stochastic optimization.

To estimate the gradient $g(\phi) = \nabla_{\phi}\text{KSD}(q_\phi \|p)^2$, we consider two unbiased stochastic gradient estimators in our implementations.
The first one is a vanilla gradient estimator using two batches of Monte Carlo samples $(x_{ri}, z_{ri})\overset{\textrm{i.i.d.}}{\sim} q_{\phi}(x,z)$ where $r=1,2$ and $1\leq i\leq N$, and it is defined as
\begin{equation}\label{eq:sto_grad}
    \hat{g}_{\textrm{vanilla}}(\phi)=\frac{1}{N^2}\sum_{1\leq i,j\leq N}\nabla_\phi \left[k(x_{1i},x_{2j})\left\langle f_{1i},f_{2j}\right\rangle \right],
\end{equation}
where $f_{ri}=s_p(x_{ri})-s_{q_\phi(\cdot|z_{ri})}(x_{ri})$.
We can also leverage U-statistics to design an alternative estimator \citep{MMD,liu2016kernelized}.
Given one batch of samples $(x_{i}, z_{i})\overset{\textrm{i.i.d.}}{\sim} q_{\phi}(x,z)$ for $1\leq i\leq N$, the U-statistic gradient estimator takes the form
\begin{equation}\label{eq:sto_grad_u}
    \hat{g}_{\textrm{u-stat}}(\phi)=\frac{2}{N(N-1)}\sum_{1\leq i<j\leq N}\nabla_\phi \left[k(x_i,x_j)\left\langle f_i,f_j\right\rangle \right],
\end{equation}
where $f_i=s_p(x_i)-s_{q_\phi(\cdot|z_i)}(x_i)$. 
Note that all the gradient terms in $\hat{g}_{\textrm{vanilla}}(\phi)$ and $\hat{g}_{\textrm{u-stat}}(\phi)$ can be efficiently evaluated using the reparameterization trick.

In our implementation, we assume a diagonal Gaussian conditional layer $q_\phi(x|z)=\mathcal{N}(\mu(z;\phi), \mathrm{diag}\{\sigma^2(z;\phi)\})$ where the mean and standard deviation $\mu(z;\phi),\sigma(z;\phi)\in \R^d$ are parametrized by neural networks.
Hence the parameterization formula of $q_\phi(x|z)$ is $x=\mu(z;\phi)+\sigma(z;\phi)\odot\xi$ with $\xi\sim q_\xi=\mathcal{N}(0,I)$, and the corresponding conditional score takes the form $s_{q_\phi(\cdot|z)}(x)=-\xi/\sigma(z;\phi)$.
Note that in principle, one can sample multiple $\xi$ for a single $z_i$ to get a more robust estimation as done in \citet{yin2018semi}.
In our experiments, we find that sampling a single $\xi$ works well.
The full training procedure of KSIVI with the vanilla gradient estimator \eqref{eq:sto_grad} is described in Algorithm \ref{alg:kernel_sivi_vanilla}.
The algorithm with U-statistic gradient estimator \eqref{eq:sto_grad_u} is deferred to Appendix \ref{app:sec:alg}.
Compared to other variants of SIVI, our proposed KSIVI optimizes the KSD objective efficiently and stably with samples from the variational distribution, without using surrogates of exact ELBO, expensive inner-loop MCMC, or additional lower-level optimization.

\begin{algorithm}[t]
    \caption{KSIVI with diagonal Gaussian conditional layer and vanilla gradient estimator}
    \label{alg:kernel_sivi_vanilla}
    \begin{algorithmic}
        \STATE{{\bfseries Input:} target score $s_p(x)$, number of iterations $T$, number of samples $N$ for stochastic gradient.}
        \STATE{{\bfseries Output:} the optimal variational parameters $\phi^\ast$.}
        \FOR{$t=0, \cdots, T-1$}
            \STATE{
            Sample $\{z_{r1},\cdots, z_{rN}\}$ from mixing distribution $q(z)$ for $r=1,2$.
            }
            \STATE{
            Sample $\{\xi_{r1},\cdots, \xi_{rN}\}$ from $\mathcal{N}(0,I)$ for $r=1,2$.
            }
            \STATE{
            Compute $x_{ri} =\mu(z_{ri};\phi)+\sigma(z_{ri};\phi)\odot \xi_{ri}$ and $f_{ri}=s_p(x_{ri})+\frac{\xi_{ri}}{\sigma(z_{ri};\phi)}$.
            }
            \STATE{
            Compute stochastic gradient $\hat{g}_{\textrm{vanilla}}(\phi)$ through \eqref{eq:sto_grad}.
            }
            \STATE{
            Set $\phi\leftarrow\text{optimizer}(\phi, \hat{g}_{\textrm{vanilla}}(\phi))$. 
            }
        \ENDFOR
        \STATE{$\phi^\ast\leftarrow\phi$}
    \end{algorithmic}
\end{algorithm}

\section{Theoretical Results}\label{sec: theory}

In this section, we provide a theoretical guarantee for the convergence of KSIVI as a black-box variational inference (BBVI) problem \citep{ranganath2014black, titsias14}.
This is non-trivial as the conventional assumptions in stochastic optimization are challenging to meet within the context of variational inference settings.
Recent works \cite{kim2023convergence, domke2023provable, kim2023linear} attempt to analyze the location-scale family, which has poor approximation capacity.
For SIVI, \citet{moens2021} assumes that $q(z)$ is a discrete measure with $n$ components and gives a sub-optimal rate depending on $\poly(n)$.
We are not aware of any analysis of the convergence of the full version SIVI with nonlinear reparameterization.

Similar to prior works, we first investigate the smoothness of the loss function and prove an upper bound of the variance of the stochastic gradient.
Then we apply the standard analysis in stochastic optimization, which ensures $\{\phi_t\}_{t\geq 1}$ converges to a stationary point.

Define the loss function as
\begin{equation}
    \mathcal{L}(\phi):=\text{KSD}(q_\phi \| p)^2=\left\|S_{q_\phi,k}\nabla\log\frac{p}{q_\phi}\right\|_{\mathcal{H}}^2.
\end{equation}
We consider a diagonal Gaussian conditional layer.
The conditional score function is $s_{q_\phi(\cdot|z)}(x) = -\frac{\xi}{\sigma(z;\phi)}$ for $x=\mu(z;\phi)+\sigma(z;\phi)\odot \xi$ with $\xi\sim \mathcal{N}(0,I)$.
We assume an unbiased gradient estimator in \eqref{eq:sto_grad} or \eqref{eq:sto_grad_u}.
An SGD update $\phi_{t+1}=\phi_t-\eta\hat{g}_t$ is then applied with learning rate $\eta$ for $T$ iterations, where $\hat{g}_t$ is the stochastic gradient in the $t$-th iteration.

\begin{assumption}\label{asp:kernel}
    There exists a constant $B>0$ such that $\forall x,y\in \R^d$, $\max\{k,\|\nabla_1 k\|, \|\nabla_{11} k\|, \|\nabla_{12} k\|\}\leq B^2$.
\end{assumption}

\begin{assumption}\label{asp:target}
    $\log p$ is three times continuously differentiable and $\|\nabla^2 \log p(x)\|\leq L$, $\|\nabla^3 \log p(x)\|\leq M$.
\end{assumption}

\begin{assumption}\label{asp:nn}
    There exists a constant $G\geq 1$ such that $\max\{\|\nabla_\phi \mu(z;\phi)\|,  \|\nabla^2_\phi \mu(z;\phi)\|\} \leq G(1+\|z\|)$.
    The same inequalities hold for $\sigma(z;\phi)$. 
    Besides, for any $z,\phi$, $\sigma(z;\phi)$ has a uniform lower bound $1/\sqrt{L}$.
\end{assumption}

As an application to location-scale family, where $\mu(z;\phi)\equiv\mu\in\R^d$, $\sigma(z;\phi)\equiv\sigma\in\R^d$, we have $G=1$.
For general nonlinear neural networks, similar first order smoothness is also assumed in the analysis of GANs \cite{arora2017generalization}. 
Note that we do not directly assume a uniformly bounded constant of smoothness of $\mu(z;\phi)$ and $\sigma(z;\phi)$, but it can grow with $\|z\|$. 
We find in practice this assumption is valid and the constant $G$ is in a reasonable range.
See Appendix \ref{app:subsec:asp} for numerical evidence.
The lower bound of $\sigma(z;\phi)$ is to avoid degeneracy and is required in \citet{domke2023provable, kim2023linear}, in which projected SGD is applied to ensure this.
Similarly, our results can be easily extended to the setting of stochastic composite optimization and  projected SGD optimizer, which we omit due to limited space.

\begin{assumption}\label{asp:moment}
    The mixing distribution $q(z)$ and the variational distribution $q_\phi(x)$  have bounded 4th-moment, i.e., $\E_{q(z)}\|z\|^4\lesssim d_z^2$, $\E_{q_\phi(x)}\|x\|^4\leq s^4$.
\end{assumption}

Overall, Assumption \ref{asp:kernel} guarantees that the kernel operator is well-defined and is easy to verify for commonly used Gaussian radial basis function (RBF) kernel and the inverse multi-quadratic (IMQ) kernel \citep{gorham2017measuring}.
Assumption \ref{asp:target} and \ref{asp:nn} ensures the smoothness of target distribution and variational distribution.
Assumption \ref{asp:moment} is a technical and reasonable assumption to bound the variance of stochastic gradient.

Based on Assumption \ref{asp:target} and \ref{asp:moment}, we can give a uniform upper bound on $\E_{q_\phi(x)}\|s_p(x)\|^4$, which is crucial to the subsequent analysis.

\begin{proposition}\label{prop:score_bound}
    Under Assumption \ref{asp:target} and \ref{asp:moment}, we have $\E_{q_\phi(x)}\|s_p(x)\|^4\lesssim L^4(s^4+\|x^*\|^4):=C^2$, where $x^*$ is any zero point of $s_p$.
\end{proposition}

\begin{theorem}\label{thm:smooth}
    The objective $\mathcal{L}(\phi)$ is $L_\phi$-smooth, where 
    \begin{equation}
        L_\phi\lesssim B^2G^2d_z\log d\left[(1\vee L)^3+Ld+ M^2 + C\right].
    \end{equation}
\end{theorem}

\begin{theorem}\label{thm:var} 
   Both gradient estimators $\hat{g}_{\textrm{vanilla}}$ and $\hat{g}_{\textrm{u\text{-}stat}}$ have bounded variance $\Sigma=\frac{\Sigma_0}{N}$, where 
    \begin{equation}
        \Sigma_0\lesssim B^4G^2d_z\log d[L^3d+L^2d^2+C^2].
    \end{equation}
\end{theorem}

\begin{theorem}\label{thm:sgd}
    Under Assumption \ref{asp:kernel}-\ref{asp:moment}, iterates from SGD update $\phi_{t+1}=\phi_t-\eta \hat{g}_t$ with proper learning rate $\eta$ include an $\varepsilon$-stationary point $\hat{\phi}$ such that $\E [\|\nabla_\phi\mathcal{L}(\hat{\phi})\|]\leq\varepsilon$, if 
    \begin{equation}
        T\gtrsim \frac{L_\phi\mathcal{L}_0}{\varepsilon^2}\left(1+\frac{\Sigma_0}{N\varepsilon^2}\right),
    \end{equation}
    where $\mathcal{L}_0:=\mathcal{L}(\phi_0)-\inf_{\phi}\mathcal{L}$.
\end{theorem}

The complete proofs are deferred to Appendix \ref{app:subsec:proof}.
Theorem \ref{thm:smooth} and \ref{thm:var} imply that the optimization problem $\min_\phi \mathcal{L}(\phi)$ satisfies the standard assumptions in non-convex optimization literature.
Therefore, classic results of SGD can be applied to our scenario \citep{ghadimi2013stochastic}.

Theorem \ref{thm:sgd} implies that the sequence $\{\phi_t\}$ converge to a stationary point of the objective $\mathcal{L}(\phi)$.  
Since the training objective is squared KSD instead of traditional ELBO in BBVI, $\mathcal{L}(\phi)$ is nonconvex even for location-scale family.
The convergence in loss function $\mathcal{L}$ is generally inaccessible.
We hope future works can shed more insights into this issue.

\section{Related Works}

To address the limitation of standard VI, implicit VI constructs a flexible variational family through non-invertible mappings parameterized by neural networks.
However, the main issue therein is density ratio estimation, which is difficult in high dimensions \cite{sugiyama2012density}.
Besides SIVI, there are a number of recent advancements in this field.
\citet{molchanov2019doubly} have further extended SIVI in the context of generative models.
\citet{sobolev2019importance} introduce a new surrogate of ELBO through importance sampling.
UIVI \cite{titsias2019unbiased} runs inner-loop MCMC to get an unbiased estimation of the gradient of ELBO.
KIVI \cite{Shi18} constrains the density ratio estimation within an RKHS.
LIVI \cite{uppal2023implicit} approximates the intractable entropy with a Gaussian distribution by linearizing the generator.

Besides ELBO-based training of VI, many works consider to minimize Fisher divergence or its variants \cite{yu2023semi, ranganath2016operator, grathwohl2020learning, dong2022particle, cheng2023particle}.
Leveraging the minimax formulation of Fisher divergence, these methods try to alternatively optimize the variational distribution and an adversarial function in certain function classes, typically neural networks.
Although it does not rely on surrogates of ELBO, the lower-level optimization can not be done accurately in general, thus inducing unavoidable bias for the training of variational distribution.
More importantly, it involves expensive extra computation during training.

To fix this computation issue, another closely related work \citep{korba2021kernel} proposes a particle-based VI method, KSDD, which follows the KSD flow to minimize KSD, a kernelized version of Fisher divergence.
It utilizes the fact that the Wasserstein gradient of KSD can be directly estimated by the empirical particle distribution.
\citet{korba2021kernel} also discusses the theoretical properties of KSDD.
However, our formulation and the "denoising" derivation are not naive extensions, since KSIVI only requires the function value of kernel, while KSDD relies on twice derivatives of the kernel.
We believe this is a crucial point that distinguishes KSIVI, since less smooth kernels can be leveraged, which may have stronger power as a distance, e.g. Riesz kernel \citep{Altekrger2023NeuralWG}.

As for the convergence guarantee of BBVI, previous works mainly focus on ELBO objectives and location-scale families.
In particular, \citet{domke2019provable} proves the bound of the gradient variance and later \citet{domke2020provable} shows the smoothness guarantee of the loss function, both of which are under various structural assumptions.
The first self-contained analysis of the convergence of location-scale family has been recently proposed by \citet{kim2023convergence, domke2023provable}, where some important variants like STL gradient estimator and proximal gradient descent are also discussed.
Still, there is no theoretical convergence guarantee of (semi) implicit VI to the best of our knowledge.

\begin{figure}[!t]
    \centering
    \includegraphics[width=\linewidth]{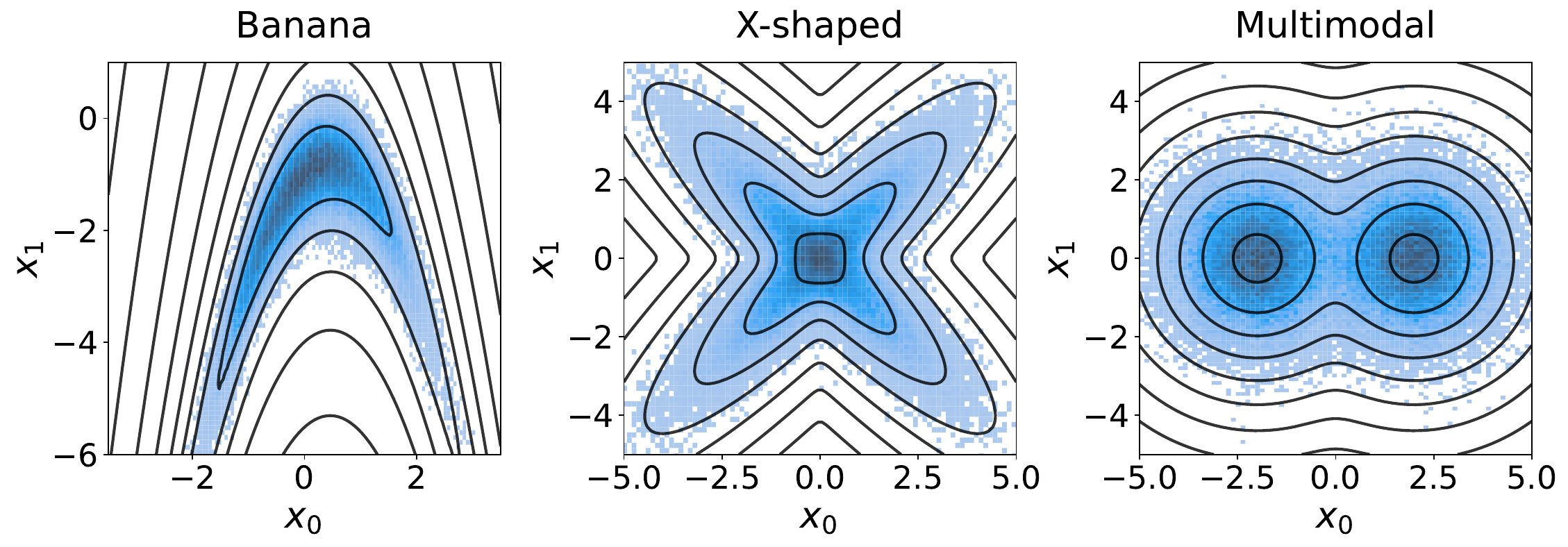}
    \caption{\textbf{Performances of KSIVI on toy examples}. The histplots in blue represent the estimated densities using 100,000 samples generated from KSIVI's variational approximation. The black lines depict the contour of the target distributions.
    }
    \label{figure: samples-toy}
\end{figure}

\begin{figure}[t]
    \centering
    \includegraphics[width=\linewidth]{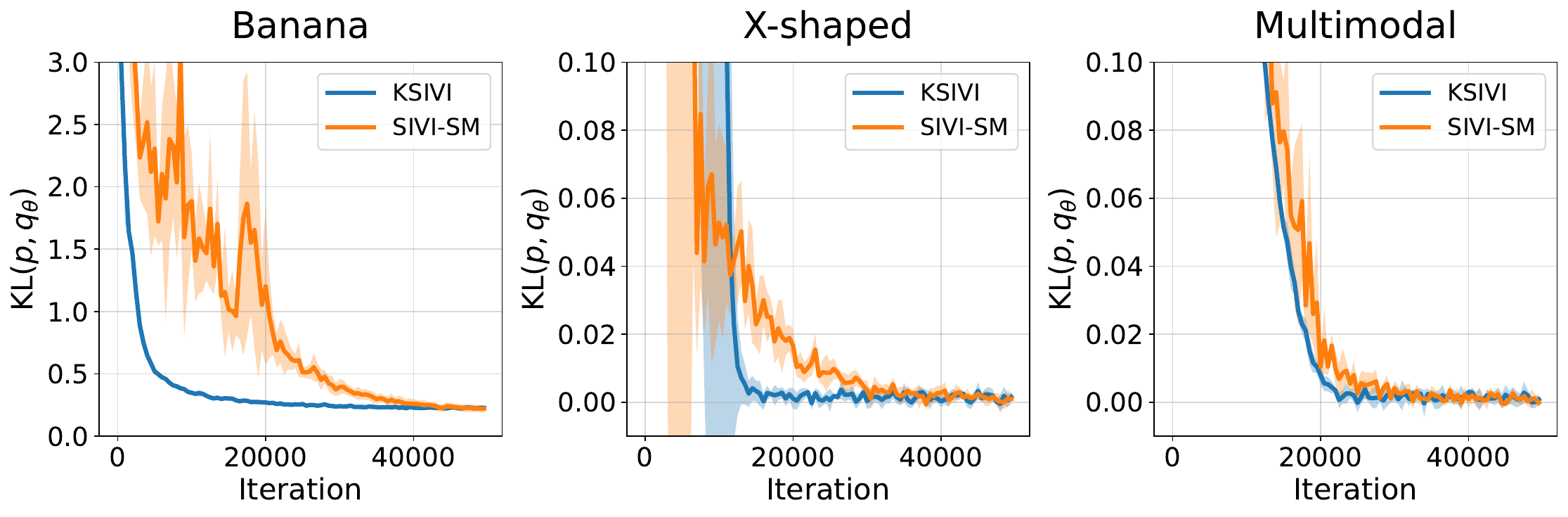}
    \caption{\textbf{Convergence of KL divergence during training obtained by different methods on toy examples.} The KL divergences are estimated using the Python ITE module \citep{ITE2014} with 100,000 samples.
    The results are averaged over 5 independent computations with the standard deviation as the shaded region.
    }
    \label{figure: kl-toy}
\end{figure}

\begin{figure*}[t]
   \centering
   \subfigure{
   \begin{minipage}[t]{0.3\linewidth}
   \centering
   \includegraphics[width=1\textwidth]{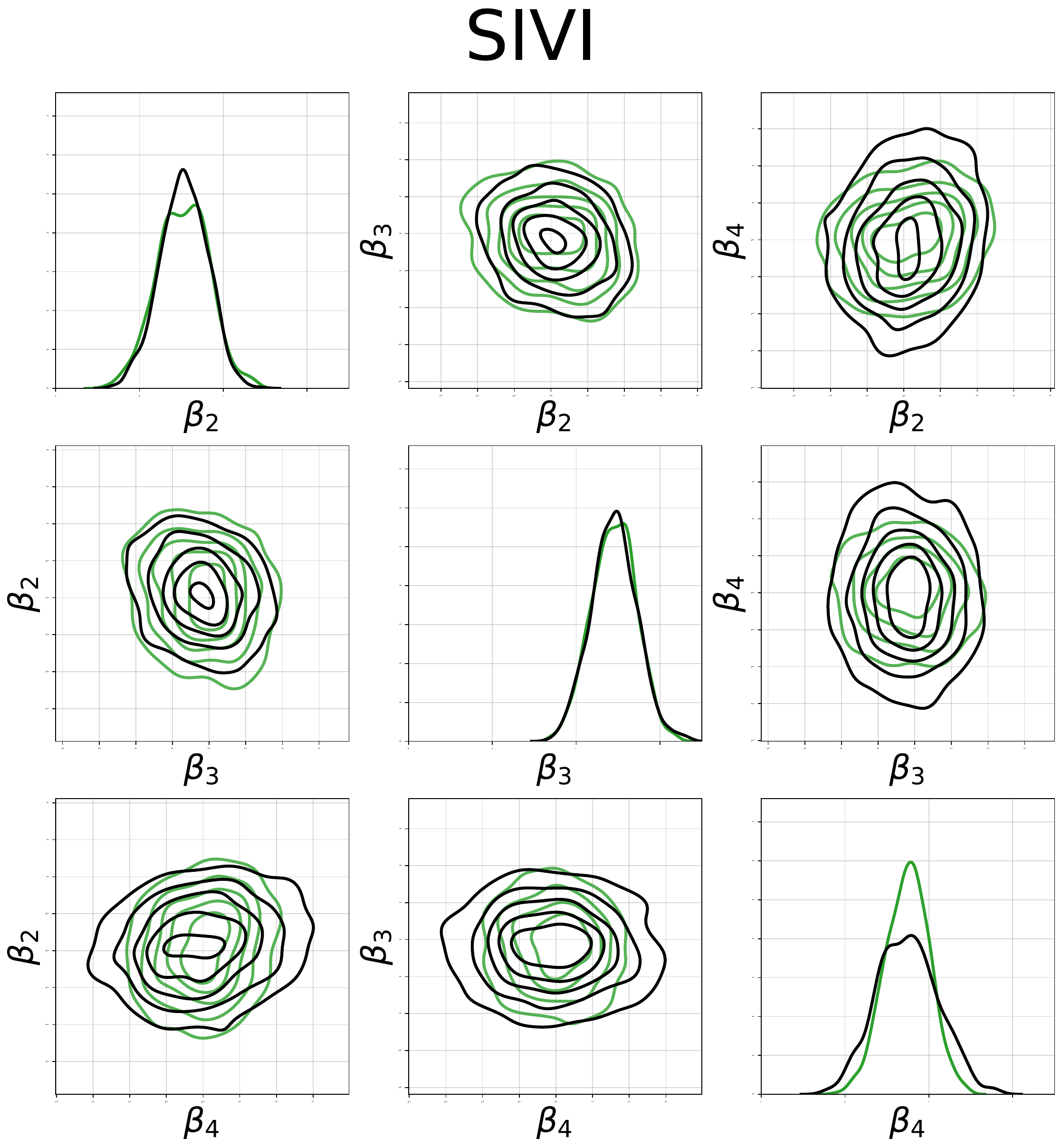}
   \end{minipage}%
   }%
   \hfill
   \subfigure{
   \begin{minipage}[t]{0.3\linewidth}
   \centering
   \includegraphics[width=1\textwidth]{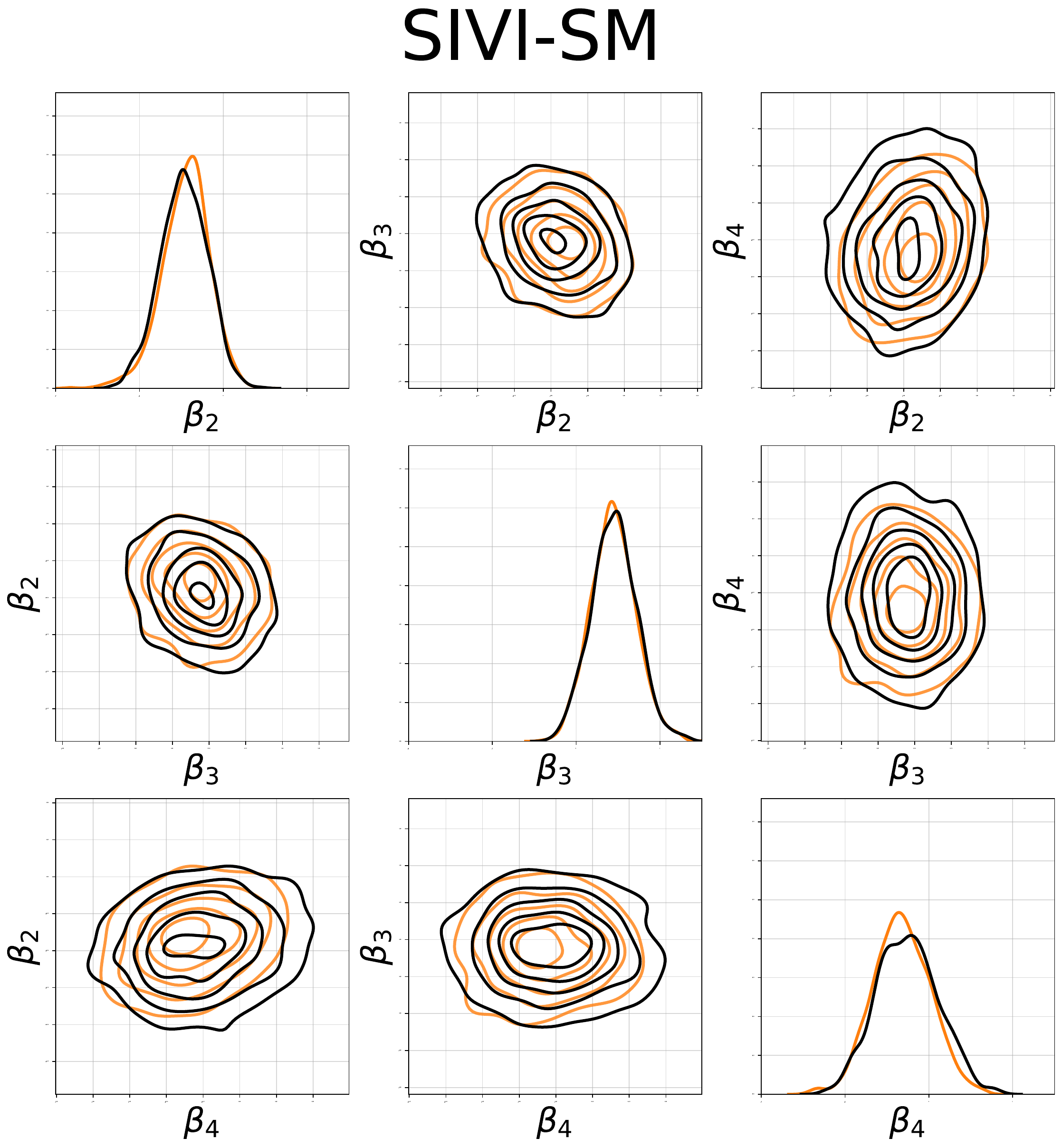}
   \end{minipage}%
   }%
   \hfill
   \subfigure{
   \begin{minipage}[t]{0.3\linewidth}
   \centering
   \includegraphics[width=1\textwidth]{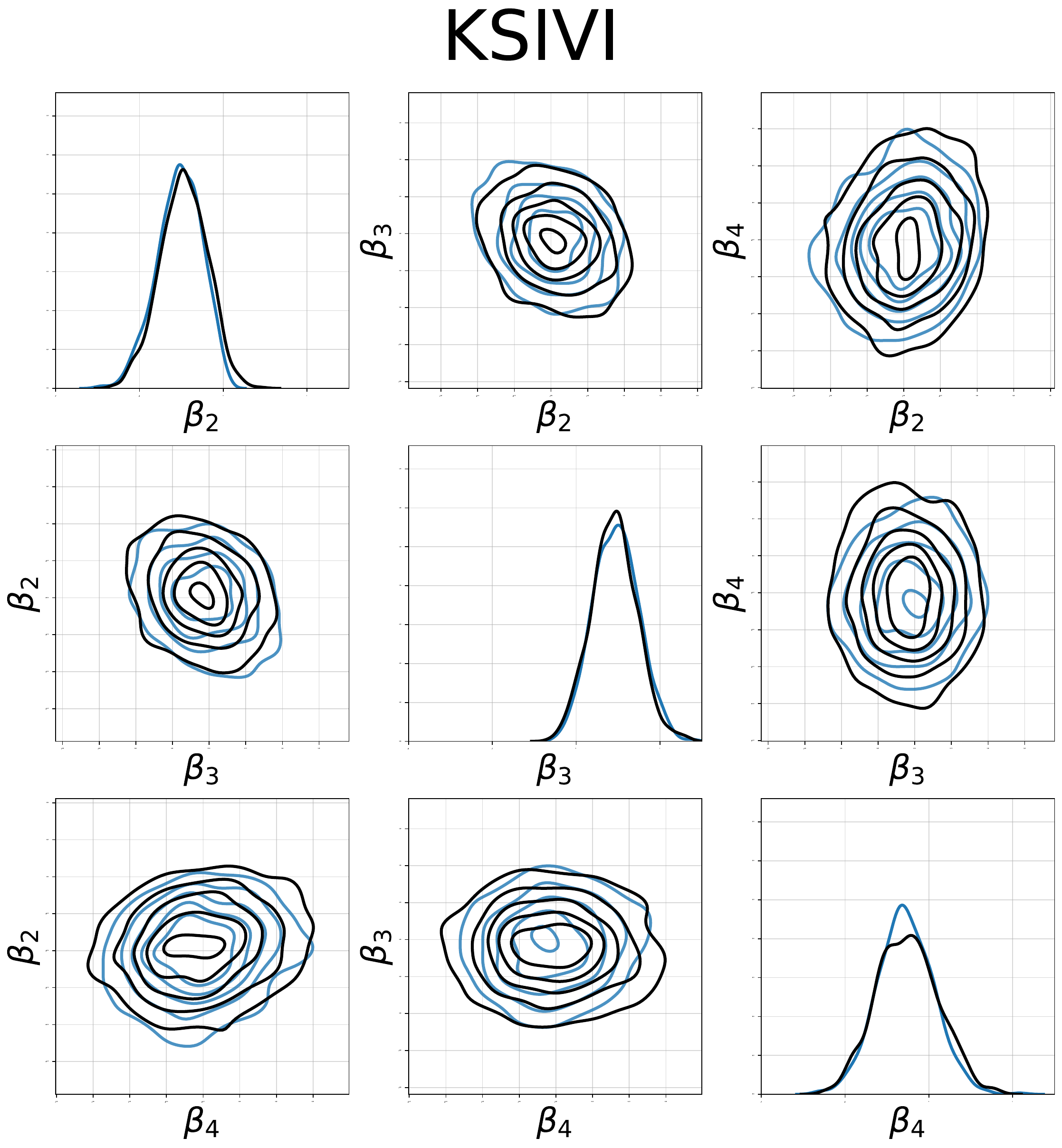}
   \end{minipage}%
   }%
   \centering
   \captionof{figure}{\textbf{Marginal and pairwise variational approximations of $\beta_2,\beta_3,\beta_4$ on the Bayesian logistic regression task}.
   The contours of the pairwise posterior approximation produced by SIVI-SM (in orange), SIVI (in green), and KSIVI (in blue) are graphed in comparison to the ground truth (in black). The sample size is 1000.
   }
   \label{figure:LRwaveform_density_dim_2_4}
\end{figure*}

\begin{figure}[t]
    \centering
    \includegraphics[width=\linewidth]{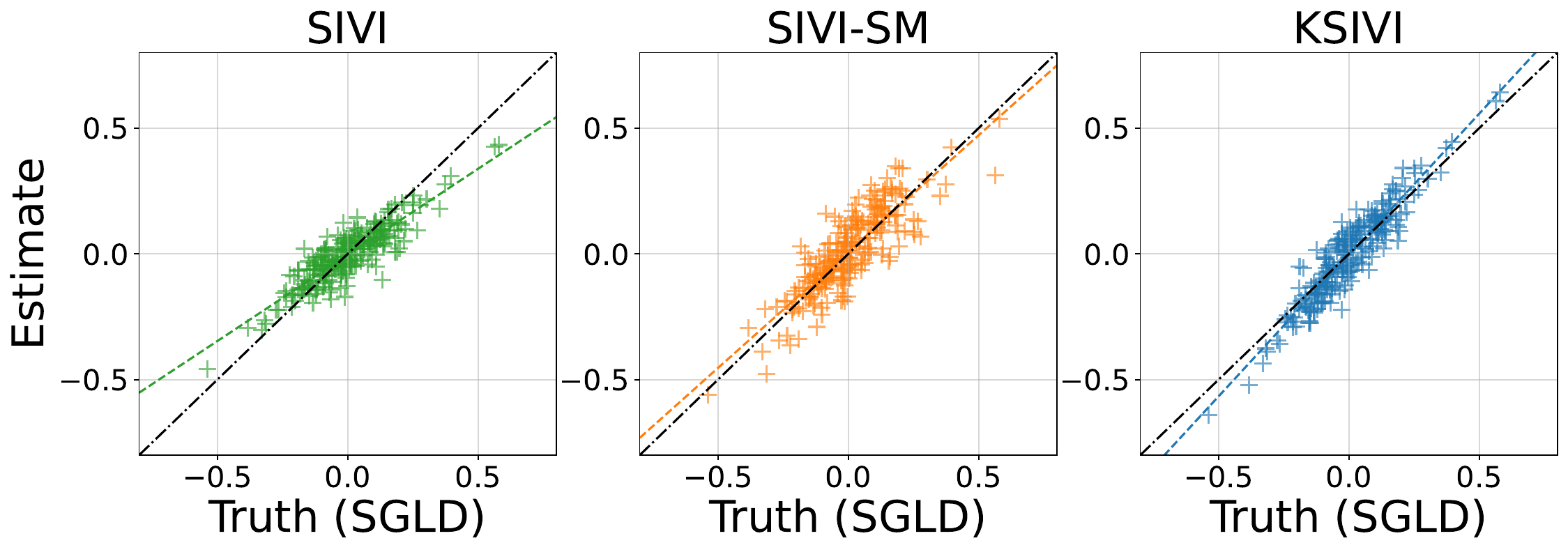}
    \caption{\textbf{Comparison between the estimated pairwise correlation coefficients and the ground truth on the Bayesian logistic regression task.} 
    Each scatter represents the estimated correlation coefficient ($y$-axis) and the ground truth correlation coefficient ($x$-axis) of some pair $(\beta_i,\beta_j)$.
    The lines in the same color as the scatters represent the regression lines.
    The sample size is 1000.
    }
    \label{figure:LRwaveform_corr}
\end{figure}

\begin{figure*}[htp]
    \centering
    \includegraphics[width=\linewidth]{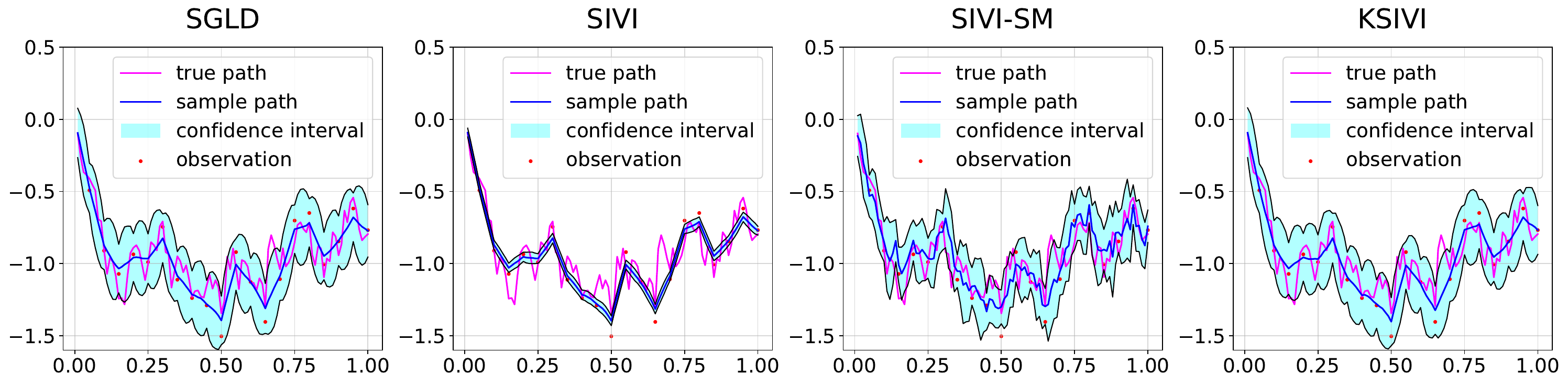}
    \caption{\textbf{Variational approximations of different methods for the discretized conditioned diffusion process.}
    The magenta trajectory represents the ground truth via parallel SGLD. The blue line corresponds to the estimated posterior mean of different methods, and the shaded region denotes the $95\%$ marginal posterior confidence interval at each time step. The sample size is 1000.
    }
    \label{figure:cd_traj}
\end{figure*}

\section{Experiments}
In this section, we compare KSIVI to the ELBO-based method SIVI and the score-based method SIVI-SM on toy examples and real-data problems. 
For the construction of the semi-implicit variational family in all these methods, we choose a standard Gaussian mixing distribution and diagonal Gaussian conditional layer (see Section \ref{sec:practical-implementation}) whose standard deviation $\sigma(z,\phi)=\phi_\sigma\in\mathbb{R}^d$ does not depend on $z$.
Following the approach by \citet{liu2016stein}, we dynamically set the kernel width, to the median value of variational samples' spacing.
In all experiments, we use the Gaussian RBF kernel in KSIVI, following \citet{liu2016stein, liu2016kernelized}. 
Throughout this section, we use the vanilla gradient estimator for KSIVI, and results of the U-statistic gradient estimator can be found in Appendix \ref{appendix:AddtionalExp}.
All the experiments are implemented in PyTorch \citep{pytorch2019}. 
More implementation details can be found in Appendix \ref{appendix:AddtionalExp} and \url{https://github.com/longinYu/KSIVI}.

\subsection{Toy Examples}
We first conduct toy experiments on approximating three two-dimensional distributions: \textsc{Banana}, \textsc{Multimodal}, and \textsc{X-shaped}, whose probability density functions are in Table~\ref{table: ToyDensity} in Appendix \ref{appendix:Toy}. 
We consider a temperature annealing strategy \citep{NF} on \textsc{Multimodal} to facilitate exploration.
The results are collected after 50,000 iterations with a learning rate of 0.001 for all the methods. 

Figure~\ref{figure: samples-toy} shows approximation performances of KSIVI on the toy examples.
We see that KSIVI provides favorable variational approximations for the target distributions. 
Figure~\ref{figure: kl-toy} displays the KL divergence as a function of the number of iterations obtained by KSIVI and SIVI-SM.
Compared to SIVI-SM, KSIVI is more stable in training, partly because it does not require additional lower-level optimization.
KSIVI also tends to converge faster than SIVI-SM, although this advantage becomes less evident as the target distribution gets more complicated.
Figure \ref{figure:mmd_toy} in Appendix \ref{appendix:Toy} illustrates the convergence of maximum mean discrepancy (MMD) during training obtained by KSIVI and SIVI-SM, which mostly aligns with the behavior of KL divergence. 
We also conduct experiments with the IMQ kernel \citep{gorham2017measuring}, and do not notice a significant difference w.r.t. the Gaussian RBF kernel, which is consistent with the findings in \citet{korba2021kernel}.

\subsection{Bayesian Logistic Regression}

Our second experiment is on the Bayesian logistic regression problem with the same experimental setting in \citet{yin2018semi}.
Given the explanatory variable $x_i\in \mathbb{R}^{d}$ and the observed binary response variable $y_i\in\{0,1\}$, the log-likelihood function takes the form
\begin{equation*}
\log p(y_i|x_i', \beta) = y_i \beta^T \bar{x}_i - \log(1+\exp(\beta^T \bar{x}_i)),
\end{equation*}
where $\bar{x}_i=\bigl[\begin{smallmatrix}1\\x_i\end{smallmatrix}\bigl] \in \mathbb{R}^{d+1}$ is the covariate and $\beta \in \mathbb{R}^{d+1}$ is the variable we want to infer. 
The prior distribution of $\beta$ is set to $p(\beta)=\mathcal{N}(0,\alpha^{-1} I)$ where the inverse variance $\alpha = 0.01$. 
We consider the \textsc{waveform}\footnote{https://archive.ics.uci.edu/ml/machine-learning-databases/waveform} dataset of $\{x_i,y_i\}_{i=1}^{N}$ where the dimension of the explanatory variable $x_i$ is $d=21$.
Then different SIVI variants with the same architecture of semi-implicit variational family are applied to infer the posterior distribution $p(\beta|\{x_i,y_i\}_{i=1}^{N})$.
The learning rate for variational parameters $\phi$ is chosen as 0.001 and the batch size of particles is chosen as 100 during the training.
For all the SIVI variants, the results are collected after 40,000 parameter updates. 
The ground truth consisting of 1000 samples is established by simulating parallel stochastic gradient Langevin dynamics (SGLD) \citep{Welling2011} with 400,000 iterations, 1000 independent particles, and a small step size of 0.0001.
Additionally, we assessed the performance of parallel Metropolis-adjusted Langevin algorithm (MALA) \citep{mala}. The calculated KL divergence between samples from MALA and SGLD is 0.0289, suggesting their proximity.

Figure~\ref{figure:LRwaveform_density_dim_2_4} demonstrates the marginal and pairwise posterior approximations for $\beta_2,\beta_3,\beta_4$ obtained by the aforementioned SIVI variants in comparison to the ground truth. 
We see that KSIVI agrees well with the ground truth and performs comparably to SIVI-SM.
Notice that SIVI-SM also requires tuning additional hyper-parameters (e.g., the learning rate of $f_\psi(x)$ and the number of lower-level gradient steps), which may be challenging as commonly observed in minimax optimization \citep{GAN,AVB}.
In contrast, SIVI slightly underestimates the variance of $\beta_4$ in both marginal and pairwise joint distributions, as shown by the left plot in Figure~\ref{figure:LRwaveform_density_dim_2_4}.
The results for more components of $\beta$ can be found in Figure~\ref{figure:LRwaveform_density_dim_1_6} in Appendix~\ref{appendix:blr}.
Additionally, we also investigate the pairwise correlation coefficients of $\beta$ defined as
\[
\bm{\rho} = \left\{\rho_{i,j}:= \frac{\mathrm{cov}(\beta_i,\beta_j)}{\sqrt{\mathrm{cov}(\beta_i,\beta_i)\mathrm{cov}(\beta_j,\beta_j)}}\right\}_{1\le i < j \le 22}
\]
and compare the estimated pairwise correlation coefficients produced by different methods to the ground truth in Figure~\ref{figure:LRwaveform_corr}. 
We see that KSIVI provides better correlation coefficient approximations than SIVI and SIVI-SM, as evidenced by the scatters more concentrated around the diagonal.

\subsection{Conditioned Diffusion Process}
Our next example is a higher-dimensional Bayesian inference problem arising from the following Langevin stochastic differential equation (SDE) with state $x_t\in \mathbb{R}$
\begin{equation}\label{cond-diffusion}
\mathrm{d} x_t = 10x_t(1-x_t^2) \mathrm{d} t + \mathrm{d} w_t, \ 0\leq t \leq 1, 
\end{equation}
where $x_0 = 0$ and $w_t$ is a one-dimensional standard Brownian motion. Equation \eqref{cond-diffusion} describes the motion of a particle with negligible mass trapped in an energy potential, with thermal fluctuations represented by the Brownian forcing \citep{detommaso2018stein, cui2016dimension, yu2023hierarchical}. 
Using the Euler-Maruyama scheme with step size $\Delta t = 0.01$, we discretize the SDE into $x=(x_{\Delta t}, x_{2\Delta t},\cdots, x_{100\Delta t})$, which defines the prior distribution $p_{\mathrm{prior}}(x)$ of the 100-dimensional variable $x$.
The perturbed 20-dimensional observation is $y=(y_{5\Delta t},y_{10\Delta t},\ldots,y_{100\Delta t})$ where $y_{5k\Delta t} \sim \mathcal{N}(x_{5k\Delta_t},\sigma^2)$ with $1\leq k\leq 20$ and $\sigma = 0.1$, which gives the likelihood function $p(y|x)$. 
Given the perturbed observations $y$, our goal is to infer the posterior of the discretized path of conditioned diffusion process $p(x|y)\propto p_{\mathrm{prior}}(x) p(y|x)$.
We simulate a long-run parallel SGLD of 100,000 iterations with 1000 independent particles and a small step size of 0.0001 to form the ground truth of 1000 samples. 
For all SIVI variants, we update the variational parameters $\phi$ for 100,000 iterations to ensure convergence (Appendix \ref{appendix:cd}). 

Table \ref{tab:run_time_cd} shows the training time per 10,000 iterations of SIVI variants on a 3.2 GHz CPU. 
For a fair time comparison, we use the score-based training of SIVI discussed in \citet{yu2023hierarchical}, which computes the $\nabla\log p(x)$ instead of $\log p(x)$ to derive the gradient estimator.
We see that KSIVI achieves better computational efficiency than SIVI-SM and comparable training time to SIVI.
Figure~\ref{figure:cd_traj} shows the approximation for the discretized conditional diffusion process of all methods.
We see that the posterior mean estimates given by SIVI-SM are considerably bumpier compared to the ground truth, while SIVI fails to capture the uncertainty with a severely underestimated variance. 
In contrast, the results from KSIVI align well with the SGLD ground truth.

\begin{table*}[t]
\caption{\textbf{Test RMSE and test NLL of Bayesian neural networks on several UCI datasets}. The results are averaged from 10 independent runs with the standard deviation in the subscripts.
  For each data set, the best result is marked in \textbf{black bold font} and the second best result is marked in \textbf{\color{Sepia!30}{brown bold font}}.
  }  
\label{tab:bnn_rmse}
\centering
\vskip0.5em
\setlength\tabcolsep{6.3pt}
\resizebox{\linewidth}{!}{
\begin{tabular}{lcccccccc}
\toprule
 \multirow{2}{*}{Dataset}&\multicolumn{4}{c}{Test RMSE ($\downarrow$)} &\multicolumn{4}{c}{Test NLL ($\downarrow$)}  \\
\cmidrule(l){2-5}\cmidrule(l){6-9}
&  SIVI& SIVI-SM & SGLD& KSIVI&SIVI& SIVI-SM & SGLD& KSIVI \\
\midrule
\textsc{Boston}       & $\bm{\textcolor{Sepia!30}{2.621}}_{\pm0.02}$ & $2.785_{\pm0.03}$ & $2.857_{\pm0.11}$& $\bm{2.555}_{\pm0.02}$& $\bm{2.481}_{\pm0.00}$ & $2.542_{\pm0.01}$ & $3.094_{\pm0.01}$& $\bm{\textcolor{Sepia!30}{2.506}}_{\pm0.01}$    \\
\textsc{Concrete}     & $6.932_{\pm0.02}$ & $\bm{\textcolor{Sepia!30}{5.973}}_{\pm0.04}$ & $6.861_{\pm0.19}$& $\bm{5.750}_{\pm0.03}$  & $3.337_{\pm0.00}$ & $\bm{3.229}_{\pm0.01}$ & $4.036_{\pm0.01}$& $\bm{\textcolor{Sepia!30}{3.309}}_{\pm0.01}$    \\
\textsc{Power}        & $\bm{3.861}_{\pm0.01}$ & $4.009_{\pm0.00}$ & $3.916_{\pm0.01}$& $\bm{\textcolor{Sepia!30}{3.868}}_{\pm0.01}$   & $\bm{2.791}_{\pm0.00}$ & $2.822_{\pm0.00}$ & $2.944_{\pm0.00}$& $\bm{\textcolor{Sepia!30}{2.797}}_{\pm0.00}$   \\
\textsc{Wine}      & $\bm{\textcolor{Sepia!30}{0.597}}_{\pm0.00}$ & $0.605_{\pm0.00}$ & $\bm{\textcolor{Sepia!30}{0.597}}_{\pm0.00}$& $\bm{0.595}_{\pm0.00}$ & $\bm{\textcolor{Sepia!30}{0.904}}_{\pm0.00}$ & $0.916_{\pm0.00}$ & $\bm{\textcolor{Sepia!30}{0.904}}_{\pm0.00}$& $\bm{0.901}_{\pm0.00}$    \\
\textsc{Yacht}        & $1.505_{\pm0.07}$ & $\bm{0.884}_{\pm0.01}$ & $2.152_{\pm0.09}$& $\bm{\textcolor{Sepia!30}{1.237}}_{\pm0.05}$  & $\bm{\textcolor{Sepia!30}{1.721}}_{\pm0.03}$ & $\bm{1.432}_{\pm0.01}$ & $2.873_{\pm0.03}$& $1.752_{\pm0.03}$    \\
\textsc{Protein}      & $\bm{4.669}_{\pm0.00}$ & $5.087_{\pm0.00}$  & $\bm{\textcolor{Sepia!30}{4.777}}_{\pm0.00}$& $5.027_{\pm0.01}$ & $\bm{2.967}_{\pm0.00}$ & $3.047_{\pm0.00}$ & $\bm{\textcolor{Sepia!30}{2.984}}_{\pm0.00}$& $3.034_{\pm0.00}$    \\
\bottomrule
\end{tabular}
}
\end{table*}

\begin{table}[t]
\caption{\textbf{Training time (per 10,000 iterations, in seconds) for the conditioned diffusion process inference task.} For all the methods, the batch size for Monte Carlo estimation is set to $N=128$.}
\label{tab:run_time_cd}
\centering
\setlength\tabcolsep{5pt}
\vskip0.5em
\begin{tabular}{lccc}
\toprule
Method \textbackslash Dimensionality & 50 & 100 & 200 \\
\midrule
SIVI & 58.00 & 88.12 & 113.46 \\
SIVI-SM & 70.42 & 128.13 & 149.61 \\
KSIVI & 56.67 & 90.48 & 107.84 \\
\bottomrule
\end{tabular}
\vspace{-1.5em}
\end{table}

\subsection{Bayesian Neural Network}
In the last experiment, we compare KSIVI against SGLD, SIVI, and SIVI-SM on sampling from the posterior of a Bayesian neural network (BNN) across various benchmark UCI datasets\footnote{https://archive.ics.uci.edu/ml/datasets.php}. 
Following \citet{liu2016stein, wang2022accelerated}, we use a two-layer neural network consisting of 50 hidden units and ReLU activation functions for the BNN model.
The datasets are randomly partitioned into 90\% for training and 10\% for testing.
Additionally, gradient clipping is applied during the training process, and the exponential moving average (EMA) trick \citep{huang2017snapshot, izmailov2019averaging} is employed in the inference stage for all SIVI variants.

The averaged test rooted mean squared error (RMSE) and negative log-likelihood (NLL) over 10 random runs are reported in Table~\ref{tab:bnn_rmse}.
We see that KSIVI can provide results on par with SIVI and SIVI-SM on all datasets.
However, it does not exhibit a significant advantage within the context of Bayesian neural networks.
This phenomenon is partially attributable to the evaluation metrics. 
Since test RMSE and test NLL are assessed from a predictive standpoint rather than directly in terms of the posterior distribution, they do not fully capture the approximation accuracy of the posterior distribution. 
Additionally, the performance of KSIVI may be contingent on the choice of kernels, particularly in higher-dimensional cases, as suggested by research on unbounded kernels \citep{Hertrich2024SlicedMMD}. 
We leave a more thorough investigation for future work.

\section{Ablation Study}
\paragraph{Kernel Choice in Heavy Tails Distribution}
Considering the pivotal role of kernel function selection in the efficacy of kernel-based methods\citep{MMDgorham17a}, we conducted an ablation study on a two-dimmensional distribution constructed as the product of two Student's t-distributions, following the setting described by \citet{li2023sampling}.
For KSIVI, we apply a practical regularization fix to thin the kernel Stein discrepancy across all the kernel functions\citep{bénard2023kernel}. 
We use the implementation from \citet{li2023sampling} to reproduce the results for two particle-based variational inference methods MIED \citep{li2023sampling} and KSDD \citep{korba2021kernel}.
In table \ref{tab:t-Student}, we report the estimated Wasserstein distances from the target distributions to the variational posteriors (using the metric implementation provided in \cite{li2023sampling} with 1000 samples).
We found that compared to the Gaussian kernel, the Reisz kernel leads to a more pronounced improvement for this task.  

\begin{table}[t]
\caption{\textbf{Estimated Wasserstein distances for the product of two Student’s t-distributions.} The results were averaged from
10 independent runs.}
\label{tab:t-Student}
\centering
\setlength\tabcolsep{5pt}
\vskip0.5em
\resizebox{\linewidth}{!}{
\begin{tabular}{lccc}
\toprule
Methods \textbackslash Edge width & 5 & 8& 10 \\
\midrule
MIED                       & $\bm{0.0366_{\pm0.01}}$ & $\bm{0.0778_{\pm0.03}}$ & $0.1396_{\pm0.04}$ \\
KSDD                       & $0.1974_{\pm 0.04}$     & $0.3187_{\pm 0.10}$     & $0.3910_{\pm 0.11}$ \\
KSIVI (Gaussian kernel)    & $0.1048_{\pm 0.03}$     & $0.1976_{\pm 0.08}$     & $0.2546_{\pm 0.09}$\\
KSIVI (Reisz kernel)       & $0.0451_{\pm 0.01}$     & $0.0816_{\pm 0.04}$     & $\bm{0.1136_{\pm 0.05}}$ \\
\bottomrule
\end{tabular}
}
\vspace{-1.5em}
\end{table}

\section{Conclusion}
This paper proposed a novel framework of semi-implicit variational inference, KSIVI, which takes the KSD as the training objective instead of the ELBO or Fisher divergence.
As a variant of SIVI-SM, KSIVI eliminates the need for lower-level optimization using the kernel trick.
By taking advantage of the hierarchical structure of the semi-implicit variational family, the KSD objective is readily computable.
We provided efficient Monte Carlo gradient estimators for stochastic optimization.
We also derived an upper bound for the variance of the Monte Carlo gradient estimators, which allows us to establish the convergence guarantee to a stationary point using techniques in stochastic optimization.
Extensive numerical experiments validate the efficiency and effectiveness of KSIVI.

\section{Limitation}
Our current work has several limitations. Firstly, since KSIVI uses KSD to measure dissimilarity between distributions, it may encounter drawbacks like stationary points that differ from the target distributions, especially in high-dimensional, non-convex tasks. Secondly, our use of the Gaussian kernel in KSIVI may be less effective in high dimensions due to its fast-decaying tails. Exploring unbounded kernel functions could be a promising direction for future research.
Additionally, our theoretical analysis depends on strong assumptions, such as the Lipschitz smoothness of the neural network, which may be too strict for deep neural networks in high dimensions. Investigating simpler kernel machines or random feature regression models as alternatives could also be valuable.

\section*{Impact Statement}
This paper presents work whose goal is to advance the field of machine learning. There are many potential societal consequences of our work, none of which we feel must be specifically highlighted here.

\section*{Acknowledgements}
This work was supported by National Natural Science Foundation of China (grant no. 12201014 and grant no. 12292983).
The research of Cheng Zhang was support in part by National Engineering Laboratory for Big Data Analysis and Applications, the Key Laboratory of Mathematics and Its Applications (LMAM) and the Key Laboratory of Mathematical Economics and Quantitative Finance (LMEQF) of Peking University.
The authors appreciate the anonymous ICML reviewers for their constructive feedback.

\bibliography{main}
\bibliographystyle{icml2024}

\newpage
\appendix
\onecolumn

\section{Proofs}\label{app:sec:proofs}

\subsection{Proof of Theorem \ref{thm:opt_f}}\label{app:subsec:proof_opt_f}
\begin{theorem}
    Consider the min-max problem
    \begin{equation}
        \min_{\phi}\max_{f} \quad \E_{q_{\phi}}\left[2f(x)^T[s_p(x) - s_{q_\phi}(x)] - \|f\|_{\mathcal{H}}^2\right].
    \end{equation}
    Given variational distribution $q_\phi$, the optimal $f^*$ is given by 
    \begin{equation}
        f^*(x)=\E_{y\sim q_\phi(y)} k(x,y)\left[s_p(y)-s_{q_\phi}(y)\right].
    \end{equation}
    Thus the upper-level problem for $\phi$ is
    \begin{equation}
        \min_\phi\quad  \text{KSD}(q_\phi\|p)^2=\left\|S_{q_\phi,k}\nabla\log\frac{p}{q_\phi}\right\|_{\mathcal{H}}^2.
    \end{equation}
\end{theorem}

\begin{proof}
    For any $f\in \mathcal{H}$, we have $f(x)=\left\langle f,k(\cdot,x)\right\rangle_{\mathcal{H}}$ by reproducing property.
    Since
    \begin{equation}
        \begin{aligned}
            \E_{q_{\phi}}f(x)^T[s_p(x) - s_{q_\phi}(x)]
            &= \E_{q_{\phi}}\left\langle f,k(\cdot,x)\right\rangle_{\mathcal{H}}^T[s_p(x) - s_{q_\phi}(x)] \\
            &= \left\langle f, \E_{q_{\phi}}k(\cdot,x)[s_p(x) - s_{q_\phi}(x)]\right\rangle_{\mathcal{H}} \\
            &= \left\langle f, S_{q_\phi,k}\nabla\log\frac{p}{q_\phi}\right\rangle_{\mathcal{H}},
        \end{aligned}
    \end{equation}
    the lower-level problem is 
    \begin{equation}
        \max_{f} \quad 2\left\langle f, S_{q_\phi,k}\nabla\log\frac{p}{q_\phi}\right\rangle_{\mathcal{H}} - \|f\|_{\mathcal{H}}^2.
    \end{equation}
    Therefore, the optimal $f^*=S_{q_\phi,k}\nabla\log\frac{p}{q_\phi}$ and the upper-level problem is 
    \begin{equation}
        \min_\phi\quad \left\|S_{q_\phi,k}\nabla\log\frac{p}{q_\phi}\right\|_{\mathcal{H}}^2.
    \end{equation}
\end{proof}

\subsection{Proofs in Section \ref{sec: theory}}\label{app:subsec:proof}

\begin{proposition}
    Under Assumption \ref{asp:target} and \ref{asp:moment}, we have $\E_{q_\phi(x)}\|s_p(x)\|^4\lesssim L^4(s^4+\|x^*\|^4):=C^2$, where $x^*$ is any zero point of $s_p$.
\end{proposition}

\begin{proof}
    By AM-GM inequality, and since $x^*$ is a zero point of $s_p$,
    \begin{equation}
        \E_{q_\phi(x)}\|s_p(x)\|^4=\E_{q_\phi(x)}\|s_p(x)-s_p(x^*)\|^4 \lesssim L^4(\E_{q_\phi(x)}\|x\|^4+\|x^*\|^4)\lesssim L^4(s^4+\|x^*\|^4).
    \end{equation}
\end{proof}

\begin{theorem}
    The objective $\mathcal{L}(\phi)$ is $L_\phi$-smooth, where $L_\phi\lesssim B^2G^2d_z\log d\left[(1\vee L)^3+Ld+ M^2 + \E \|s_p(x)\|^2\right]$.
\end{theorem}

\begin{proof}
    Let $u=(z,\xi)$ and $x=T(u;\phi)=\mu(z;\phi)+\sigma(z;\phi)\odot \xi$.
    Then the objective has the following representation:
    \begin{equation}
        \begin{aligned}
            \mathcal{L}
            &= \E_{z,z'}\E_{q_\phi(x|z),q_\phi(x'|z')} \left[k(x,x')(s_p(x)-s_{q_\phi(\cdot|z)}(x))^T(s_p(x')-s_{q_\phi(\cdot|z')}(x'))\right]\\
            &= \E_{u,u'} \left[k(x,x')(s_p(x)+\xi/\sigma)^T(s_p(x')+\xi'/\sigma')\right],
        \end{aligned}
    \end{equation}
    which implies
    \begin{equation}
        \begin{aligned}
            \nabla_\phi^2 \mathcal{L}
            &= \E \left\{\underbrace{\nabla_\phi^2[k(x,x')] (s_p(x)+\xi/\sigma)^T(s_p(x')+\xi'/\sigma')}_{\cirone}\right. \\
            &\quad + \underbrace{\nabla_\phi k(x,x')\otimes \nabla_\phi \left[(s_p(x)+\xi/\sigma)^T(s_p(x')+\xi'/\sigma')\right] + \nabla_\phi \left[(s_p(x)+\xi/\sigma)^T(s_p(x')+\xi'/\sigma')\right]\otimes \nabla_\phi k(x,x')}_{\cirtwo} \\
            &\quad + \left.\underbrace{k(x,x')\nabla_\phi^2 \left[(s_p(x)+\xi/\sigma)^T(s_p(x')+\xi'/\sigma')\right]}_{\cirthree}\right\}.
        \end{aligned}
    \end{equation}

    For $\cirone$, we have 
    \begin{equation}
        \begin{aligned}
            \nabla_\phi^2[k(x,x')] 
            &= \nabla_\phi^2x\nabla_1k + \nabla_\phi^T x\left(\nabla_{11}k\nabla_\phi x+\nabla_{21}k\nabla_\phi x'\right) \\
            &+ \nabla_\phi^2x'\nabla_2k + \nabla_\phi^T x'\left(\nabla_{22}k\nabla_\phi x'+\nabla_{12}k\nabla_\phi x\right).
        \end{aligned}
    \end{equation}
    Here $\nabla_\phi^2x\nabla_1k$ is a matrix with entry $[\nabla_\phi^2x\nabla_1k]_{ij}=\sum_{l=1}^d\nabla_{\phi}^2x_l\nabla_{x_l}k$.
    
    Note that by Assumption \ref{asp:nn}$, \|\nabla_\phi x\|\leq \|\nabla_\phi \mu\|+\|\nabla_\phi \sigma\odot\xi\|\leq G(1+\|z\|)(1+\|\xi\|_\infty)$ and $\|\nabla_\phi^2 x\|\leq \|\nabla_\phi^2 \mu\|+\|\nabla_\phi^2 \sigma\odot\xi\|\leq G(1+\|z\|)(1+\|\xi\|_\infty)$. 
    Then by Assumption \ref{asp:kernel}, $\|\nabla_1 k\|\leq B^2, \|\nabla_{11} k\|\leq B^2$,
    \begin{equation}
        \begin{aligned}
            \E [\|\cirone\|]
            &\leq \E[\|\nabla_\phi^2 k\|\cdot\|s_p(x)+\xi/\sigma\|\|s_p(x')+\xi'/\sigma'\|] \\
            &\leq \sqrt{\E[\|\nabla_\phi^2 k\|^2]}\sqrt{\E [\|s_p(x)+\xi/\sigma\|^2\|s_p(x')+\xi'/\sigma'\|^2]} \\
            &\lesssim
            \left[GB^2\sqrt{\E (1+\|\xi\|_\infty)^2(1+\|z\|)^2} + G^2B^2 \sqrt{\E (1+\|\xi\|_\infty)^4(1+\|z\|)^4}\right] \E [\|s_p(x)+\xi/\sigma\|^2] \\
            &\lesssim G^2B^2d_z\log d \E[\|s_p(x)+\xi/\sigma\|^2] \\
            &\lesssim G^2B^2d_z\log d [\E\|s_p(x)\|^2+Ld].
        \end{aligned}
    \end{equation}
    In the inequality second to last, we utilize the well-known fact $\E \|\xi\|_\infty^4\lesssim \log^2 d$ and Assumption \ref{asp:moment}.
    
    For $\cirtwo$, we have
    \begin{equation}
        \begin{aligned}
            \|\nabla_\phi k(x,x')\|
            &\leq \|\nabla_\phi^Tx\nabla_1k\|+\|\nabla_\phi^Tx' \nabla_2k\| \\
            &\lesssim B^2G[(1+\|z\|)(1+\|\xi\|_\infty) + (1+\|z'\|)(1+\|\xi'\|_\infty)],
        \end{aligned}
    \end{equation}
    and
    \begin{equation}
        \begin{aligned}
            \nabla_\phi \left[(s_p(x)+\xi/\sigma)^T(s_p(x')+\xi'/\sigma')\right]
            &= \left[\nabla^2\log p(x)\nabla_\phi x-\diag\{\xi/\sigma^2\}\nabla_\phi\sigma\right]^T(s_p(x')+\xi'/\sigma') \\
            &\quad + \left[\nabla^2\log p(x')\nabla_\phi x'-\diag\{\xi'/\sigma'^2\}\nabla_\phi\sigma'\right]^T(s_p(x)+\xi/\sigma).
        \end{aligned}
    \end{equation}
    Therefore,
    \begin{equation}
        \begin{aligned}
            \E [\|\cirtwo\|]
            &\lesssim B^2G\sqrt{\E [(1+\|z\|)^2(1+\|\xi\|_\infty)^2]}\sqrt{\E\left[\|\left[\nabla^2\log p(x)\nabla_\phi x-\diag\{\xi/\sigma^2\}\nabla_\phi\sigma\right]^T(s_p(x')+\xi'/\sigma')\|^2\right]} \\
            &\lesssim B^2G\sqrt{d_z\log d }\cdot LG\sqrt{\E [(1+\|z\|)^2(1+\|\xi\|_\infty)^2\|s_p(x')+\xi'/\sigma'\|^2]} \\
            &\lesssim LG^2B^2d_z\log d \sqrt{\E\|s_p(x)\|^2+Ld}.
        \end{aligned}
    \end{equation}

    For $\cirthree$, we have
    \begin{equation}
        \begin{aligned}
            \nabla_\phi^2\left[(s_p(x)+\xi/\sigma)^T(s_p(x')+\xi'/\sigma')\right]
            &= \nabla_\phi^2\left[s_p(x)^Ts_p(x')+\frac{\xi}{\sigma}\cdot s_p(x')+\frac{\xi'}{\sigma'}\cdot s_p(x) + \frac{\xi}{\sigma}\cdot \frac{\xi'}{\sigma'}\right]
        \end{aligned}
    \end{equation}
    Firstly,
    \begin{equation}
        \begin{aligned}
            \nabla_\phi^2 [s_p(x)^Ts_p(x')]
            &= \nabla_\phi^2x [\nabla^2\log p(x)s_p(x')] + \nabla_\phi^T x [\nabla^3\log p(x)s_p(x')]\nabla_\phi x \\
            &\qquad + \nabla_\phi^2x' [\nabla^2\log p(x')s_p(x)] + \nabla_\phi^T x' [\nabla^3\log p(x')s_p(x)]\nabla_\phi x' \\
            &\qquad + \nabla_\phi^T x \nabla^2\log p(x)\nabla^2\log p(x')\nabla_\phi x' + \nabla_\phi^T x' \nabla^2\log p(x')\nabla^2\log p(x)\nabla_\phi x.
        \end{aligned}
    \end{equation}
    \begin{equation}
        \begin{aligned}
            \E \left[k(x,x')\|\nabla_\phi^2 [s_p(x)^Ts_p(x')]\|\right]
            &\lesssim B^2\E \left[\left[LG(1+\|z\|)(1+\|\xi\|_\infty) + MG^2(1+\|z\|)^2(1+\|\xi\|_\infty)^2\right]\|s_p(x')\|\right] \\
            &\qquad + B^2L^2G^2\E[(1+\|z\|)(1+\|\xi\|_\infty)(1+\|z'\|)(1+\|\xi'\|_\infty)] \\
            &\lesssim GB^2(L\sqrt{d_z\log d}+GMd_z\log d)\sqrt{\E \|s_p(x)\|^2}+B^2L^2G^2d_z\log d.
        \end{aligned}
    \end{equation}
    Then, since
    \begin{equation}
        \nabla_\phi^2 [\frac{\xi}{\sigma}\cdot s_p(x')]
        = \nabla_\phi^2[s_p(x')] \frac{\xi}{\sigma} + \nabla_\phi^2[\frac{\xi}{\sigma}]s_p(x') + \nabla_\phi[s_p(x')]\otimes \nabla_\phi[\frac{\xi}{\sigma}] + \nabla_\phi[\frac{\xi}{\sigma}] \otimes\nabla_\phi[s_p(x')],
    \end{equation}
    it holds that
    \begin{equation}
        \begin{aligned}
            \|\nabla_\phi^2 [\frac{\xi}{\sigma}\cdot s_p(x')]\|
            &\lesssim \|\nabla_\phi^2[s_p(x')]\|\cdot\|\frac{\xi}{\sigma}\|+ \|\nabla_\phi^2[\frac{\xi}{\sigma}]\|\cdot \|s_p(x')\| + \|\nabla_\phi[\frac{\xi}{\sigma}]\| \cdot\|\nabla_\phi[s_p(x')]\| \\
            &\lesssim \left[M\|\nabla_\phi x'\|^2+L\|\nabla_\phi^2 x'\|\right]\cdot L^{1/2}\|\xi\| + (1\vee L)^{3/2}G^2(1+\|z\|)^2\|\xi\|_\infty\cdot\|s_p(x')\| \\
            &\qquad + L\|\nabla_\phi x'\| \cdot LG(1+\|z\|)\|\xi\|_\infty,
        \end{aligned}
    \end{equation}
    and thus
    \begin{equation}
        \begin{aligned}
            \E \left[k(x,x')\|\nabla_\phi^2 [\frac{\xi}{\sigma}\cdot s_p(x')]\|\right]
            &\lesssim B^2 G(L\sqrt{d_z\log d}+GMd_z\log d)L^{1/2}d^{1/2} \\
            &\qquad + B^2G^2(1\vee L)^{3/2}d_z\log^{1/2} d\sqrt{\E \|s_p(x)\|^2} + B^2G^2L^2d_z\log d.
        \end{aligned}
    \end{equation}
    Lastly, we have
    \begin{equation}
        \nabla_\phi^2[\frac{\xi}{\sigma}\cdot \frac{\xi'}{\sigma'}] = \nabla_\phi^2[\frac{\xi}{\sigma}]\frac{\xi'}{\sigma'} + \nabla_\phi^2[\frac{\xi'}{\sigma'}]\frac{\xi}{\sigma} + \nabla_\phi[\frac{\xi'}{\sigma'}]\otimes \nabla_\phi[\frac{\xi}{\sigma}] + \nabla_\phi[\frac{\xi}{\sigma}] \otimes\nabla_\phi[\frac{\xi'}{\sigma'}].
    \end{equation}
    And
    \begin{equation}
        \begin{aligned}
            \E \left[k(x,x')\|\nabla_\phi^2[\frac{\xi}{\sigma}\cdot \frac{\xi'}{\sigma'}]\|\right]
            &\lesssim B^2\E \left[(1\vee L)^{3/2}G^2(1+\|z\|)^2\|\xi\|_\infty\cdot L^{1/2}\|\xi'\| + \|\xi\|_\infty LG(1+\|z\|) \cdot \|\xi'\|_\infty LG(1+\|z'\|) \right] \\
            &\lesssim B^2 \left[(1\vee L)^{3/2}L^{1/2}G^2d_zd^{1/2}\log^{1/2}d + L^2G^2d_z\log d\right] \\
            &\lesssim B^2 (1\vee L)^{3/2}L^{1/2}G^2d_zd^{1/2}\log^{1/2} d.
        \end{aligned}
    \end{equation}
    Therefore, we get
    \begin{equation}
        \begin{aligned}
            \E[\|\cirthree\|]
            &\lesssim B^2G^2d_z\log d\left((1\vee L)^{3/2}+M\right)\sqrt{\E \|s_p(x)\|^2+Ld}
        \end{aligned}
    \end{equation}
    Combining $\cirone,\cirtwo,\cirthree$, we can conclude that 
    \begin{equation}
        \|\nabla_\phi^2\mathcal{L}\|\lesssim B^2G^2d_z\log d\left[(1\vee L)^3+Ld+ M^2 + \E \|s_p(x)\|^2\right].
    \end{equation}
\end{proof}

\begin{theorem}
    Both gradient estimators $\hat{g}_{\textrm{vanilla}}$ and $\hat{g}_{\textrm{u\text{-}stat}}$ have bounded variance:
    \begin{equation}
        \Var (\hat{g})\lesssim \frac{B^4G^2d_z\log d[L^3d+L^2d^2+\E \|s_p(x)\|^4]}{N}.
    \end{equation}
\end{theorem}

\begin{proof}
    It is sufficient to look at the upper bound of variance with two samples.
    Note that 
    \begin{equation}
        \begin{aligned}
            \nabla_\phi \left[k(x,x')(s_p(x)+\xi/\sigma)^T(s_p(x')+\xi'/\sigma') \right]
            &= \underbrace{\nabla_\phi [k(x,x')](s_p(x)+\xi/\sigma)^T(s_p(x')+\xi'/\sigma')}_{\cirone} \\
            &\quad + \underbrace{k(x,x')\nabla_\phi\left[(s_p(x)+\xi/\sigma)^T(s_p(x')+\xi'/\sigma')\right]}_{\cirtwo}.
        \end{aligned}
    \end{equation}

    For $\cirone$, we again utilize the fact that
    \begin{equation}
        \|\nabla_\phi k(x,x')\| \lesssim B^2G[(1+\|z\|)(1+\|\xi\|_\infty) + (1+\|z'\|)(1+\|\xi'\|_\infty)],
    \end{equation}
    and thus
    \begin{equation}
        \begin{aligned}
            \E [\|\cirone\|^2] 
            &\lesssim \E[\|\nabla_\phi k\|^2 \|s_p(x)+\xi/\sigma\|^2\|s_p(x')+\xi'/\sigma'\|^2] \\
            &\lesssim B^4G^2d_z\log d[\E\|s_p(x)+\xi/\sigma\|^4]^{1/2} \E\|s_p(x)+\xi/\sigma\|^2\\
            &\lesssim B^4G^2d_z\log d \left[\sqrt{\E\|s_p(x)\|^4}+Ld\right]\E\|s_p(x)+\xi/\sigma\|^2.
        \end{aligned}
    \end{equation}

    For $\cirtwo$, we have 
    \begin{equation}
        \begin{aligned}
            \nabla_\phi \left[(s_p(x)+\xi/\sigma)^T(s_p(x')+\xi'/\sigma')\right]
            &= \left[\nabla^2\log p(x)\nabla_\phi x-\diag\{\xi/\sigma^2\}\nabla_\phi\sigma\right]^T(s_p(x')+\xi'/\sigma') \\
            &\quad + \left[\nabla^2\log p(x')\nabla_\phi x'-\diag\{\xi'/\sigma'^2\}\nabla_\phi\sigma'\right]^T(s_p(x)+\xi/\sigma).
        \end{aligned}
    \end{equation}
    Hence
    \begin{equation}
        \begin{aligned}
            \E [\|\cirtwo\|^2]
            &\lesssim B^4\E\left[\left\|\left[\nabla^2\log p(x)\nabla_\phi x-\diag\{\xi/\sigma^2\}\nabla_\phi\sigma\right]^T(s_p(x')+\xi'/\sigma')\right\|^2\right] \\
            &\lesssim B^4 L^2G^2\E [(1+\|z\|)^2(1+\|\xi\|_\infty)^2\|s_p(x')+\xi'/\sigma'\|^2] \\
            &\lesssim B^4L^2G^2d_z\log d \left[\E\|s_p(x)+\xi/\sigma\|^2\right].
        \end{aligned}
    \end{equation}
    Therefore, combining $\cirone, \cirtwo$, we conclude that
    \begin{equation}
        \begin{aligned}
            &\E \left[\left\|\nabla_\phi \left[k(x,x')(s_p(x)+\xi/\sigma)^T(s_p(x')+\xi'/\sigma') \right]\right\|^2\right] \\
            &\qquad\qquad \lesssim B^4G^2d_z\log d\left[Ld+L^2+\sqrt{\E \|s_p(x)\|^4}\right]\left[\E\|s_p(x)+\xi/\sigma\|^2\right],
        \end{aligned}
    \end{equation}
    which implies
    \begin{equation}
        \begin{aligned}
            \Var (\hat{g})
            &\lesssim \frac{1}{N}\Var\left(\left[k(x,x')(s_p(x)+\xi/\sigma)^T(s_p(x')+\xi'/\sigma') \right]\right)\\
            &\lesssim \frac{B^4G^2d_z\log d[L^3d+L^2d^2+\E \|s_p(x)\|^4]}{N}.
        \end{aligned}
    \end{equation}
\end{proof}

\begin{theorem}
Under Assumption \ref{asp:kernel}-\ref{asp:moment}, the iteration sequence generated by SGD $\phi_{t+1}=\phi_t-\eta \hat{g}_t$ with proper learning rate $\eta$ includes an $\varepsilon$-stationary point $\hat{\phi}$ such that $\E [\|\nabla_\phi\mathcal{L}(\hat{\phi})\|]\leq\varepsilon$, if 
\begin{equation}
    T\gtrsim \frac{L_\phi\mathcal{L}_0}{\varepsilon^2}\left(1+\frac{\Sigma_0}{N\varepsilon^2}\right).
\end{equation}
Here $\mathcal{L}_0:=\mathcal{L}(\phi_0)-\inf_{\phi}\mathcal{L}$. 
\end{theorem}

\begin{proof}
    Since $\mathcal{L}(\cdot)$ is $L_\phi$-smooth and $\hat{g}_t$ is unbiased with $\Var(\hat{g}_t)\leq \frac{\Sigma_0}{N}$, we have
    \begin{equation}
        \begin{aligned}
            \E_t\mathcal{L}(\phi_{t+1}) 
            &\leq E_t \left[\mathcal{L}(\phi_t) + \left\langle \nabla_\phi \mathcal{L}(\phi_t), \phi_{t+1}-\phi_t \right\rangle + \frac{L_\phi}{2}\|\phi_{t+1}-\phi_t\|^2\right] \\
            &\leq \mathcal{L}(\phi_t) - (\eta-\frac{\eta^2L_\phi}{2}) \|\nabla_\phi \mathcal{L}(\phi_t)\|^2 + \frac{\eta^2L_\phi\Sigma_0}{2N}.
        \end{aligned}
    \end{equation}
    Let $\eta\leq 1/L_\phi$.
    Take full expectation and sum from $0$ to $T-1$.
    \begin{equation}
        \begin{aligned}
            \frac{1}{T}\sum_{t=0}^{T-1}\E[\|\nabla_\phi\mathcal{L}(\phi_t)\|^2]
            &\lesssim \frac{\mathcal{L}(\phi_0)-\E \mathcal{L}(\phi_T)}{\eta T} + \frac{\eta L_\phi\Sigma_0}{N}\\
            &\lesssim \sqrt{\frac{\Sigma_0 L_\phi \mathcal{L}_0}{NT}} + \frac{L_\phi \mathcal{L}_0}{T}.
        \end{aligned}
    \end{equation}
    In the last inequality, we plug optimal $\eta\asymp \min\left\{1/L_\phi, \sqrt{\frac{N\mathcal{L}_0}{L_\phi \sigma_0T}}\right\}$.
    We conclude the proof by substituting the RHS with $\varepsilon^2$ and solving $T$.
\end{proof}

\subsection{Justification for Assumption \ref{asp:nn}}\label{app:subsec:asp}

In all our experiments, the variance $\sigma(\cdot;\phi)$ is a constant vector $\sigma\in \R^d$, rather than a neural network.
Therefore the smoothness assumption for $\sigma(\cdot;\phi)$ is valid for $G=1$.
As for $\mu(\cdot;\phi)$, we compute $\|\nabla_\phi \mu(z;\phi)\|$ and $\|\nabla^2_\phi \mu(z;\phi)\|$ for randomly sampled $z$ and $\phi$ in the training process.
Note that for simplicity, we compute the Frobeneious norm and estimate the Lipschitz constant of Jacobian $\nabla_\phi \mu(z;\phi)$, which are larger than the operator norms. 
We plot the result in BLR experiment in Figure \ref{fig:smooth_mu} and \ref{fig:smooth_mu_u_stats}, showing that $G$ remains within a certain range during the early, middle, and late stages of training.
Therefore, Assumption \ref{asp:nn} is reasonable.

\begin{figure}[ht]
    \centering
    \subfigure{
        \begin{minipage}{0.45\linewidth}
            \centering
            \includegraphics[width=1\linewidth]{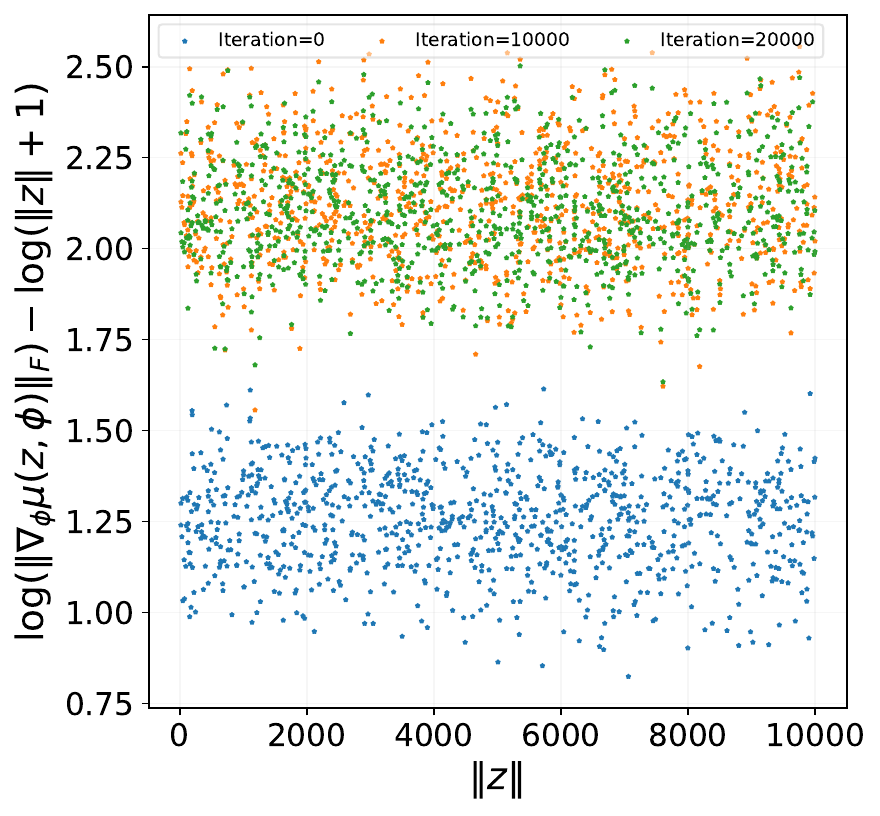}
        \end{minipage}
    }
    \hfill
    \subfigure{
        \begin{minipage}{0.45\linewidth}
            \centering
            \includegraphics[width=1\linewidth]{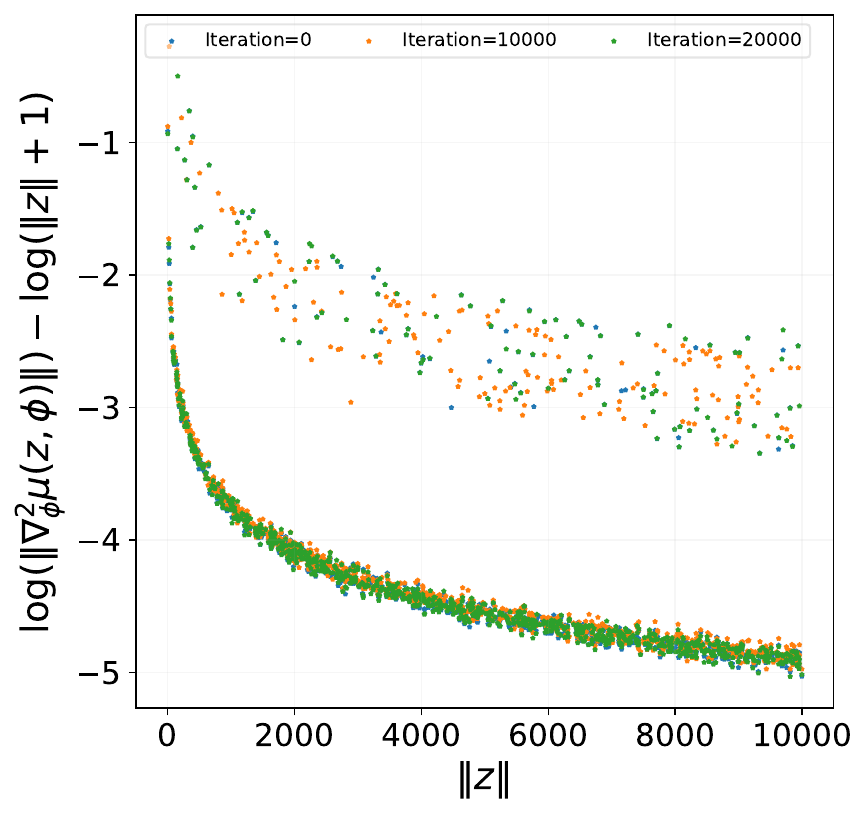}
        \end{minipage}
    }
    \centering
    \caption{Smoothness of neural network $\mu(\cdot;\phi)$ in KSIVI ($\hat{g}_{\textrm{vanilla}}$) trajectory.}
    \label{fig:smooth_mu}
\end{figure}

\begin{figure}[ht]
    \centering
    \subfigure{
        \begin{minipage}{0.45\linewidth}
            \centering
            \includegraphics[width=1\linewidth]{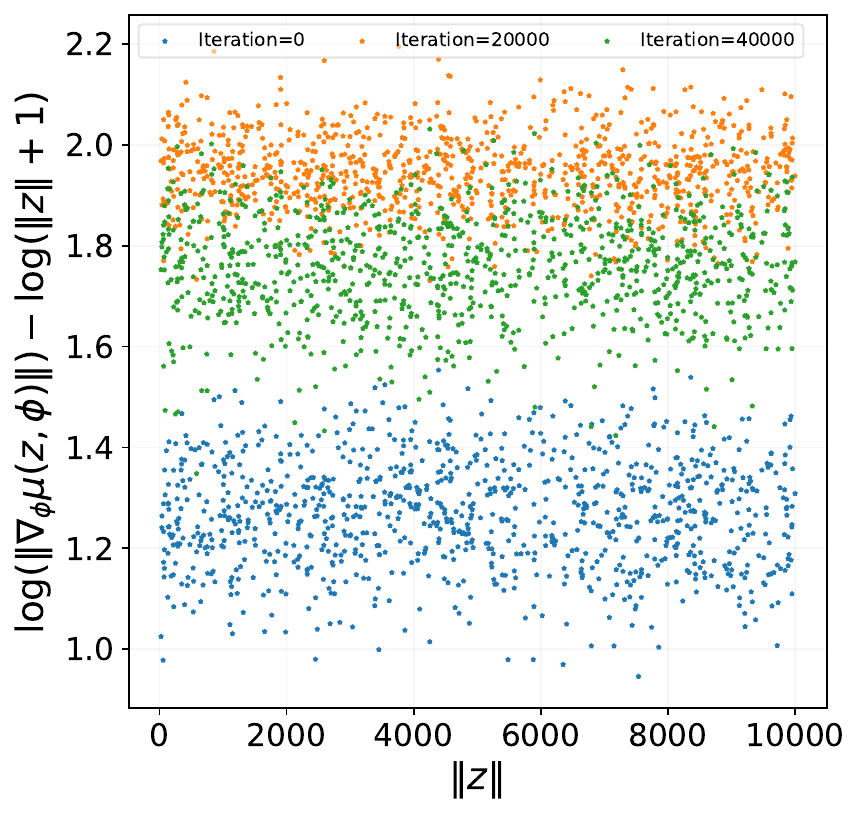}
        \end{minipage}
    }
    \hfill
    \subfigure{
        \begin{minipage}{0.45\linewidth}
            \centering
            \includegraphics[width=1\linewidth]{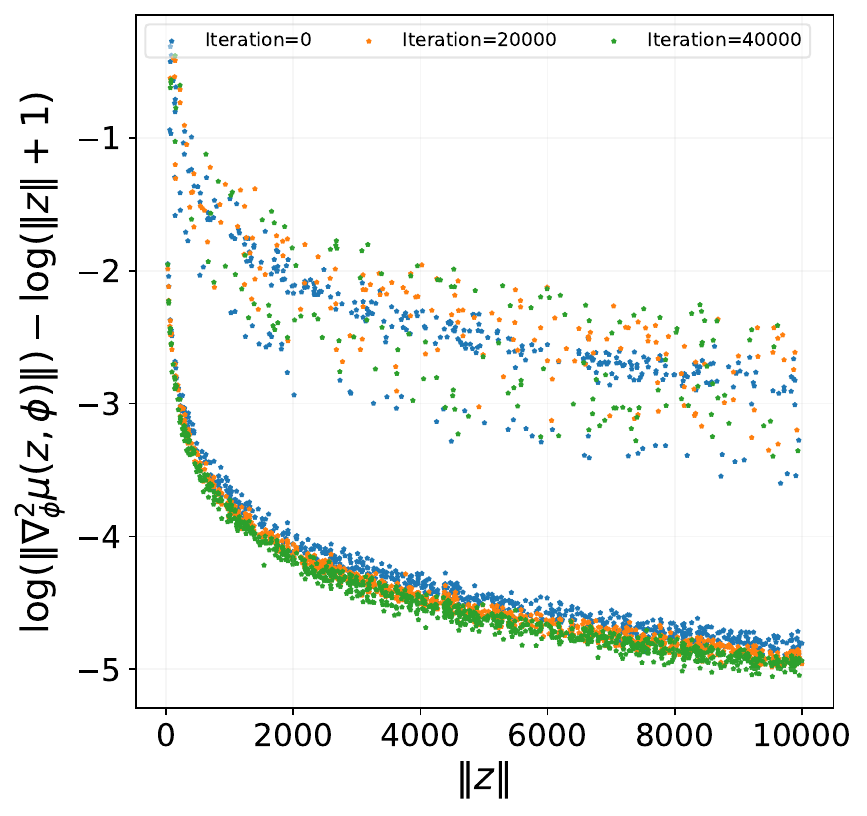}
        \end{minipage}
    }
    \centering
    \caption{Smoothness of neural network $\mu(\cdot;\phi)$ in KSIVI ($\hat{g}_{\textrm{u\text{-}stat}}$) trajectory.}
    \label{fig:smooth_mu_u_stats}
\end{figure}

\section{KSIVI Algorithm}\label{app:sec:alg}

\begin{algorithm}[H]
    \caption{KSIVI with diagonal Gaussian conditional layer and U-statistic gradient estimator}
    \label{alg:kernel_sivi}
    \begin{algorithmic}
        \STATE{{\bfseries Input:} target score $s_p(x)$, number of iterations $T$, number of samples $N$ for stochastic gradient.}
        \STATE{{\bfseries Output:} the optimal variational parameters $\phi^\ast$.}
        \FOR{$t=0, \cdots, T-1$}
            \STATE{
            Sample $\{z_1,\cdots, z_N\}$ from mixing distribution $q(z)$;
            }
            \STATE{
            Sample $\{\xi_1,\cdots, \xi_N\}$ from $\mathcal{N}(0,I)$;
            }
            \STATE{
            Compute $x_i =\mu(z_i;\phi)+\sigma(z_i;\phi)\odot \xi_i$ and $f_i=s_p(x_i)+\xi_i/\sigma(z_i;\phi)$;
            }
            \STATE{
            Compute the gradient estimate $\hat{g}_{\textrm{u-stat}}(\phi)$ through \eqref{eq:sto_grad_u};
            }
            \STATE{
            Set $\phi\leftarrow\text{optimizer}(\phi, \hat{g}_{\textrm{u-stat}}(\phi))$;
            }
        \ENDFOR
        \STATE $\phi^\ast\leftarrow\phi$.
    \end{algorithmic}
\end{algorithm}

\section{Experiment Setting and Additional Results}\label{appendix:AddtionalExp}
\subsection{Toy Experiments}\label{appendix:Toy}
\paragraph{Setting details} In the three 2-D toy examples, we set the conditional layer to be $\mathcal{N}(\mu(z;\phi),\mathrm{diag}\{\phi_\sigma^2\}$. 
The $\mu(z;\phi)$ in SIVI-SM and KSIVI all have the same structures of multi-layer perceptrons (MLPs) with layer widths $[3, 50, 50, 2]$ and ReLU  activation functions. 
The initial value of $\phi_\sigma$ is set to $1$ except for banana distribution on which we use $\frac12$. 
We set the learning rate of variational parameters $\phi$ to 0.001 and the learning rate of $\psi$ in SIVI-SM to 0.002. 
In the lower-level optimization of SIVI-SM, $\psi$ is updated after each update of $\phi$.
\paragraph{Additional results} Figure \ref{figure:mmd_toy} depicts the descent of maximum mean discrepancy (MMD) during the training process of SIVI-SM and KSIVI.

\begin{figure}[t]
    \centering
    \includegraphics[width=0.95\linewidth]{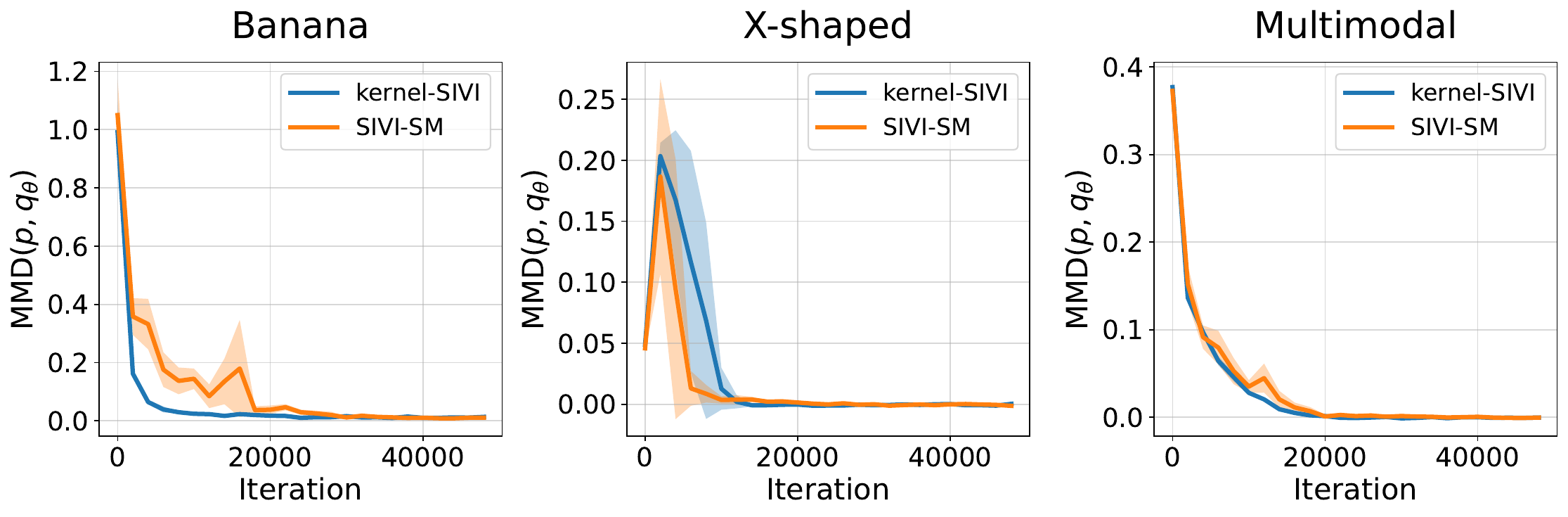}
    \caption{Convergence of MMD divergence of different methods on 2-D toy examples. 
    The MMD objectives are estimated using 1000 samples.
    The results are averaged over 5 independent computations with the standard deviation as the shaded region.
    }
    \label{figure:mmd_toy}
\end{figure}

\begin{table}[H]
   \captionof{table}{Three 2-D target distributions implemented in the 2-D toy experiments.}
   \label{table: ToyDensity}
   \begin{center}
   \begin{tabular}{cll}
   \toprule
   \multicolumn{1}{c}{Name}  & \multicolumn{2}{c}{ Density}\\
   \midrule
   \textsc{Banana} & $x = \left(v_1, v_1^2 + v_2 + 1\right)^T$, $v \sim \mathcal{N}(0, \Sigma)$ & $\Sigma = \bigl[\begin{smallmatrix}1 & 0.9\\0.9 & 1\end{smallmatrix}\bigl]$\\[2ex] 
   \textsc{Multimodal}    & $x\sim\frac{1}{2}\mathcal{N}(x|\mu_1, I)+\frac{1}{2}\mathcal{N}(x|\mu_2, I)$ & $\mu_1 = [-2, 0]^T, \mu_2 = [2, 0]^T$\\[2ex]
   \textsc{X-shaped}      & $ x\sim\frac{1}{2}\mathcal{N}(x|0, \Sigma_1)+\frac{1}{2}\mathcal{N}(x|0, \Sigma_2)$ & $\Sigma_1 = \bigl[\begin{smallmatrix}2 & 1.8\\1.8 & 2\end{smallmatrix}\bigl], \Sigma_2 = \bigl[\begin{smallmatrix} 2 & -1.8\\-1.8 & 2\end{smallmatrix}\bigl]$\\[1ex]
   \bottomrule
   \end{tabular}
   \end{center}
\end{table}

\subsection{Bayesian Logistic Regression}\label{appendix:blr}
\paragraph{Setting details} For the experiment on Bayesian logistic regression, $\mu(z;\phi)$ for all the SIVI variants are parameterized as MLPs with layer widths $[10, 100, 100, 22]$ and ReLU activation functions. 
For SIVI-SM, the $f_\psi$ is three-layer MLPs with 256 hidden units and ReLU activation functions. 
The initial value of $\phi_\sigma^2$ is $e^{-5}$ for KSIVI and 1 for SIVI and SIVI-SM. 
These initial values are determined through grid search on the set $\{e^{0}, e^{-2}, e^{-5}\}$.
For SIVI, the learning rate of $\phi_\mu$ is set to 0.01, and the learning rate of $\phi_\sigma$ is set to 0.001, following the choices made in \citet{yin2018semi}. 
The learning rate for the variational parameters $\phi$, comprising $\phi_\mu$ and $\phi_\sigma$, is set to 0.001 for both SIVI and SIVI-SM.
The learning rate of $\psi$ for SIVI-SM is set to 0.002, and we update $\psi$ after each update of $\phi$.
For SIVI, the auxiliary sample size is set to $K = 100$.
For all methods, we train the variational distributions for 20,000 iterations with a batch size of 100.

\paragraph{Additional results} In line with Figure 5 in \citet{Domke18}, we provide the results of KSIVI using varying step sizes of $0.00001$, $0.0001$, $0.001$, $0.005$, and $0.008$.
The results in Table \ref{tab:LRwaveform_wd} indicate that KSIVI demonstrates more consistent convergence behavior across different step sizes.
Figure \ref{figure:LRwaveform_density_dim_1_6} shows the marginal and pairwise joint density of the well-trained posterior $q_\phi(\beta_1, \beta_2, \cdots, \beta_6)$ for SIVI, SIVI-SM, KSIVI with vanilla gradient estimator and KSIVI with U-statistic gradient estimator. 
Figure \ref{figure:LRwaveform_corr_u} presents the results of the estimated pairwise correlation coefficients by KSIVI using the U-statistic gradient estimator.

\begin{table}[t]
    \caption{\textbf{Estimated sliced Wasserstein distances (Sliced-WD) of KSIVI and SIVI-SM on Bayesian logistic regression.} The Sliced-WD objectives are estimated using 1000 samples.}
    \label{tab:LRwaveform_wd}
    \centering
    \setlength\tabcolsep{5pt}
    \vskip0.5em
    \begin{tabular}{lccccc}
    \toprule
    Methods \textbackslash Step sizes & 0.00001 & 0.0001 & 0.001 & 0.005 & 0.008 \\
    \midrule
    SIVI-SM (iteration = 10k) & 2.5415 & 24.2952  & 0.1203 & 0.6107 & 0.9906\\
    SIVI-SM (iteration = 20k) & 5.6184 & 28.6335  & \textbf{0.0938} & 0.3283 & 0.9145\\
    KSIVI (iteration = 10k) &1.5196& 0.6863 & 0.1509 &\textbf{0.1064} & 0.6059\\
    KSIVI (iteration = 20k) &\textbf{1.2600}&\textbf{0.1120}&0.0965&0.2628&\textbf{0.3186} \\
    \bottomrule
    \end{tabular}
    \vspace{-1.5em}
    \end{table}

\begin{figure}[t]
   \centering
   \subfigure{
   \begin{minipage}[t]{0.47\linewidth}
   \centering
   \includegraphics[width=1\textwidth]{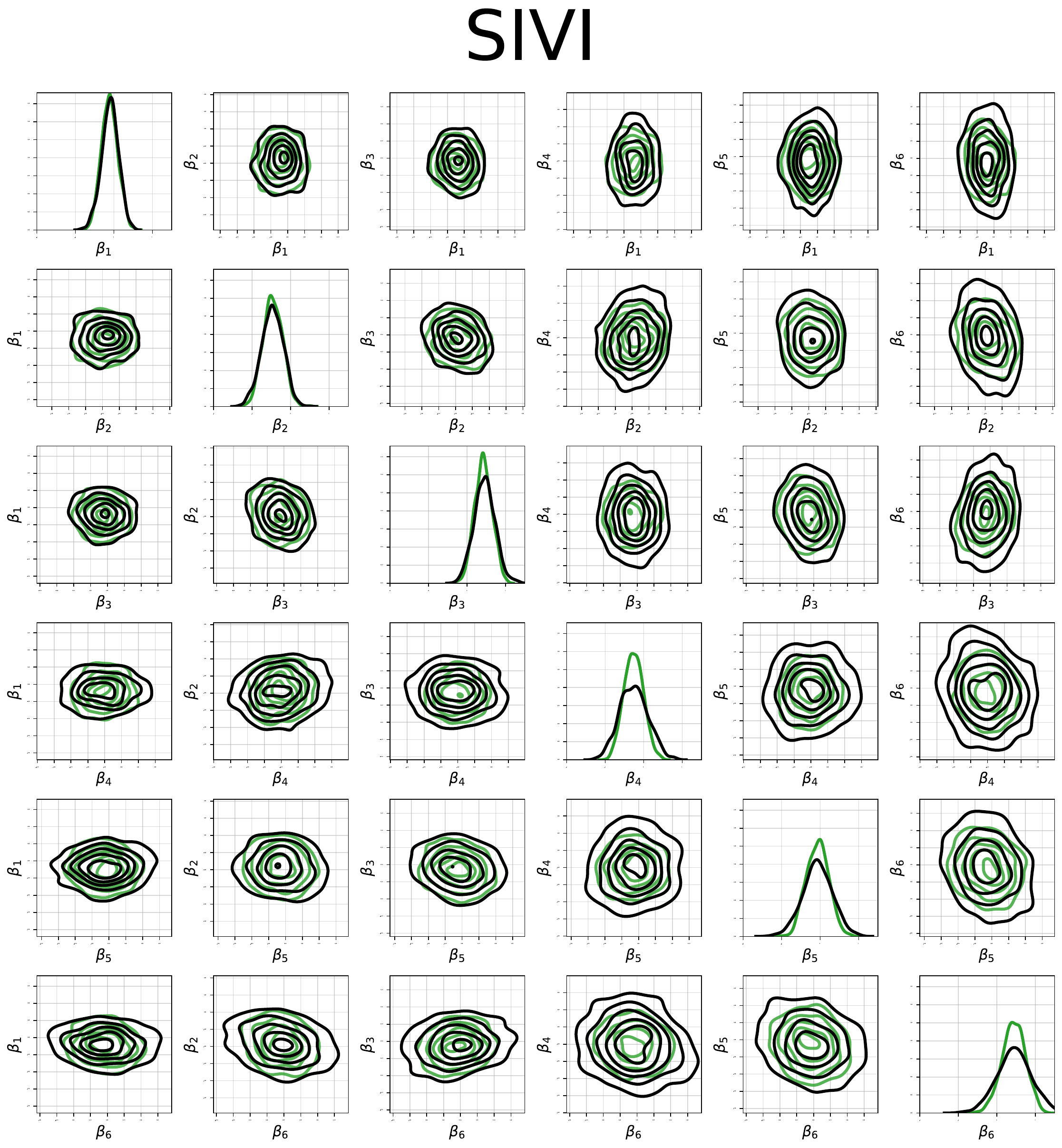}
   \end{minipage}
   }
   \hfill
   \subfigure{
   \begin{minipage}[t]{0.47\linewidth}
   \centering
   \includegraphics[width=1\textwidth]{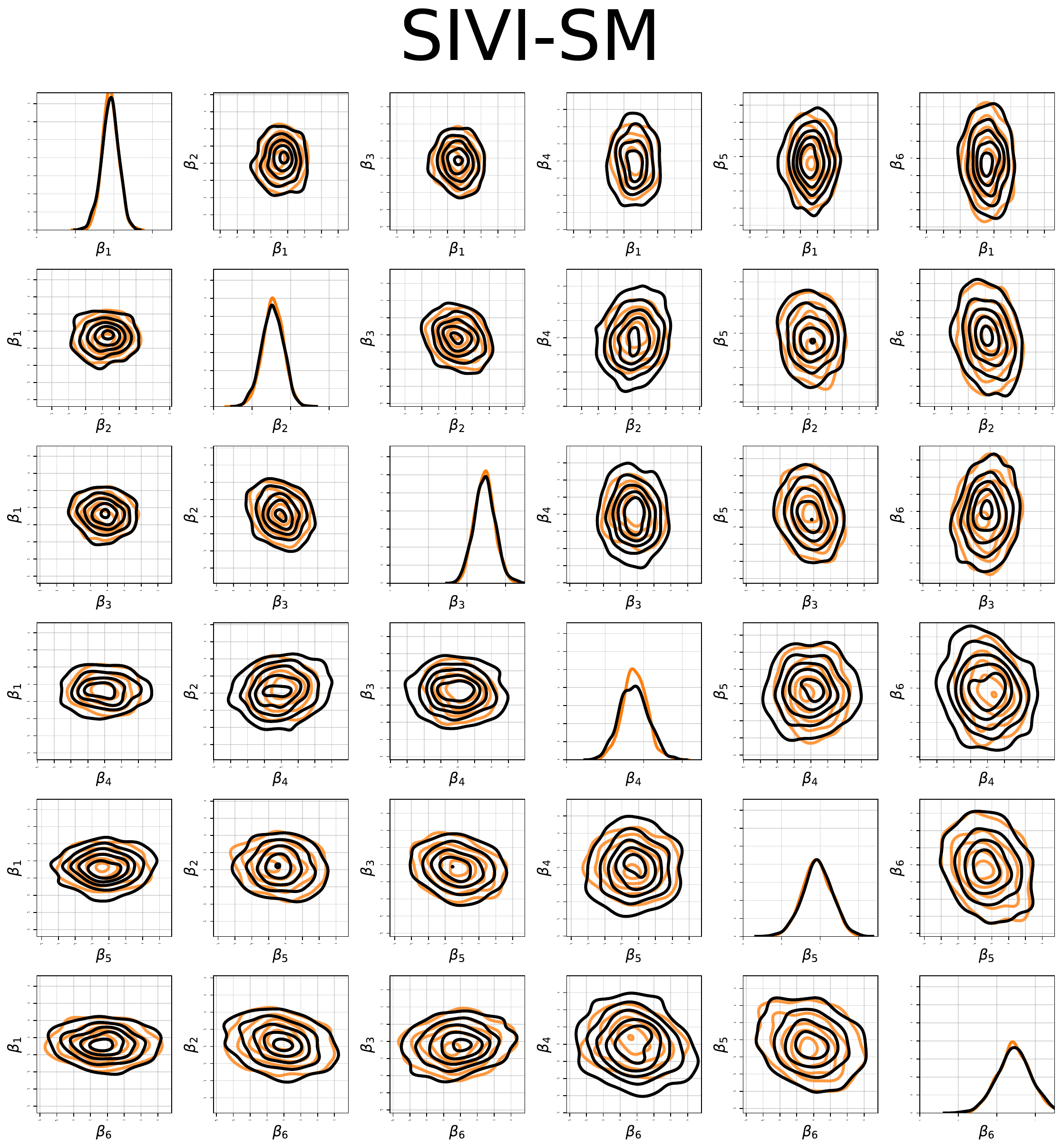}
   \end{minipage}
   }
   \hfill
   \subfigure{
   \begin{minipage}[t]{0.47\linewidth}
   \centering
   \includegraphics[width=1\textwidth]{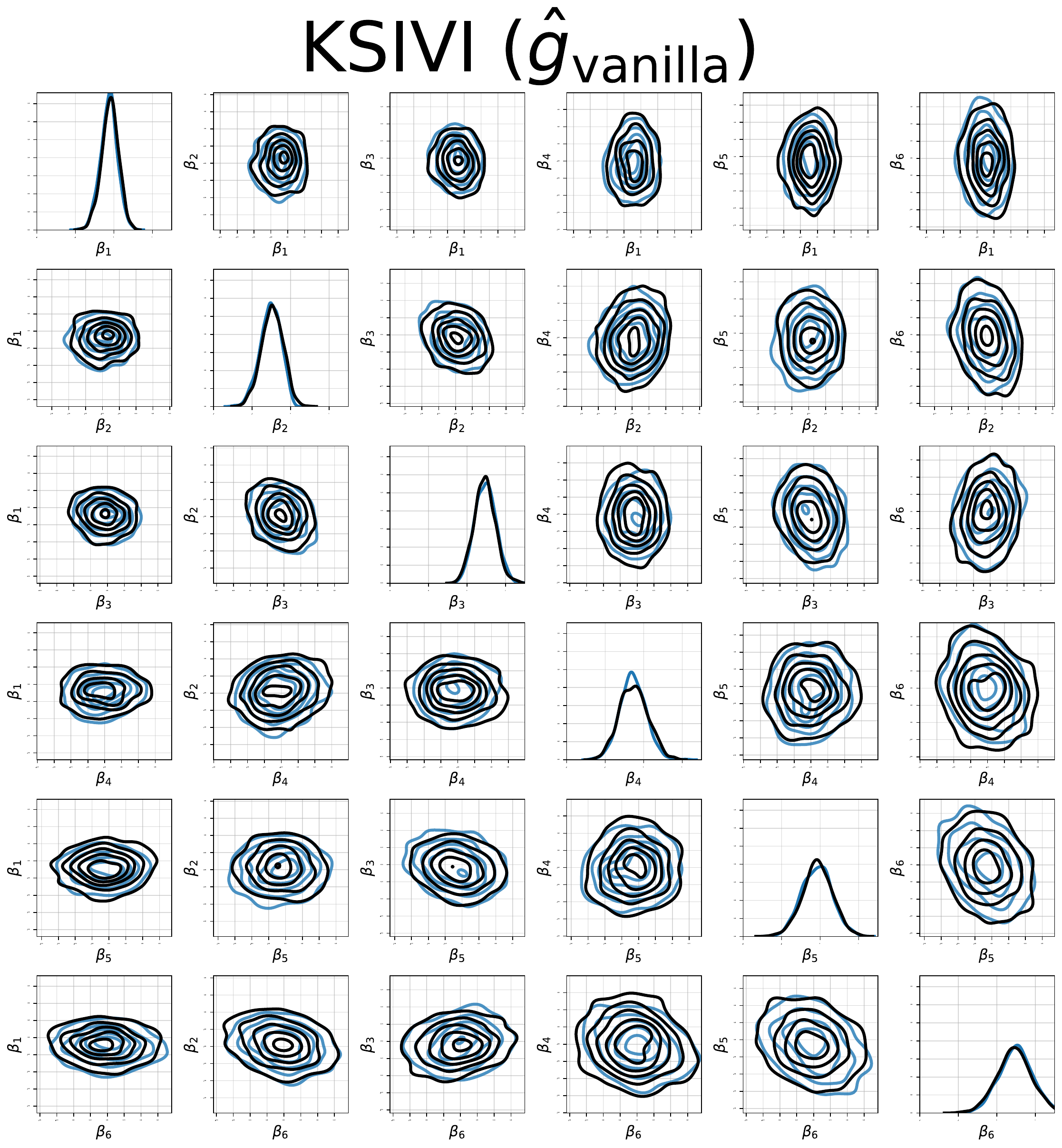}
   \end{minipage}
   }
   \hfill
   \subfigure{
   \begin{minipage}[t]{0.47\linewidth}
   \centering
   \includegraphics[width=1\textwidth]{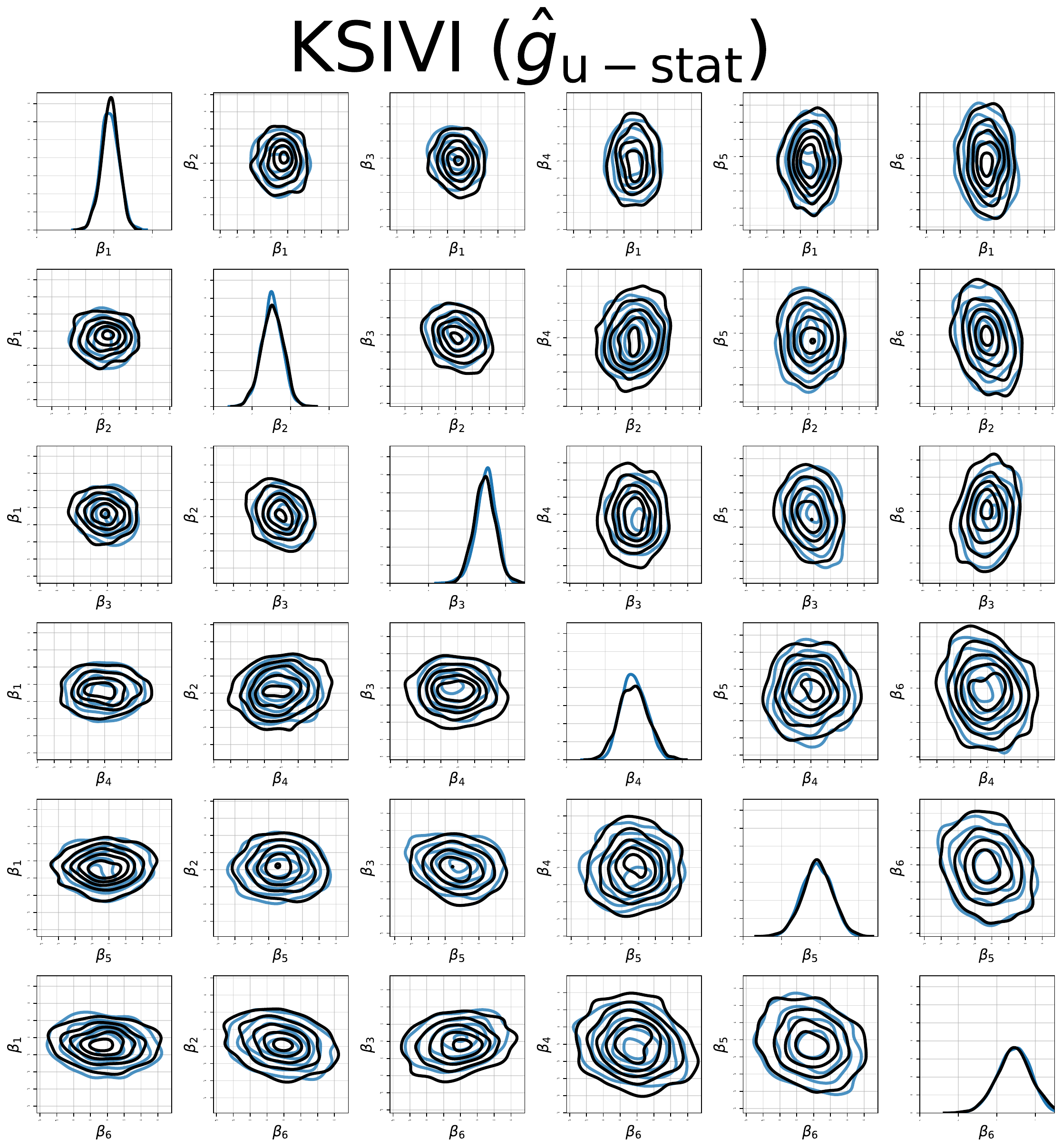}
   \end{minipage}
   }
   \centering
   \captionof{figure}{\textbf{Marginal and pairwise joint posteriors on $\bm{\beta_1,\cdots,\beta_6}$}.
   The contours of the pairwise empirical densities produced by three baseline algorithms, i.e. SIVI-SM (in orange), SIVI (in green), and KSIVI (in blue) including both KSIVI ($\hat{g}_{\textrm{vanilla}}$) and KSIVI ($\hat{g}_{\textrm{u-stat}}$) are graphed in comparison to the ground truth(in black).
   }
   \label{figure:LRwaveform_density_dim_1_6}
\end{figure}

\begin{figure}[t]
    \centering
    \includegraphics[width=\linewidth]{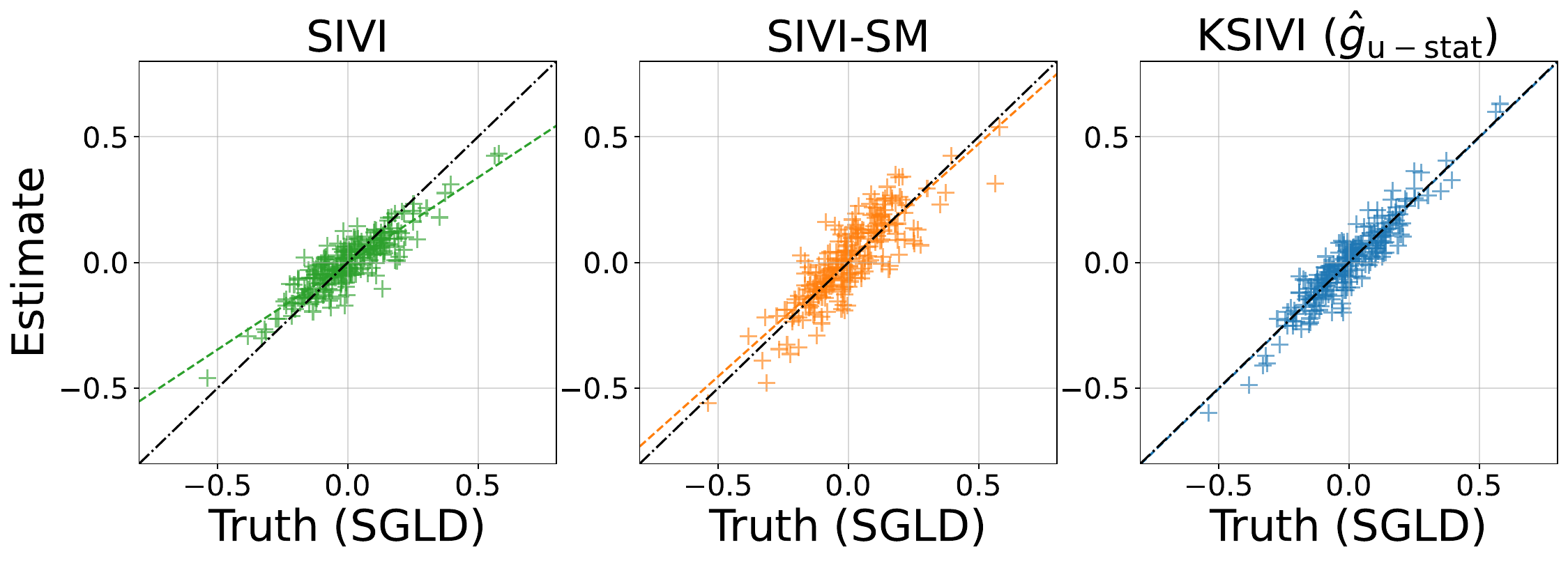}
    \caption{\textbf{Comparison result between the estimated pairwise correlation coefficients on SIVI, SIVI-SM, and KSIVI with U-statistic gradient estimator.} 
    The correlation coefficients are estimated with 1000 samples.
    }
    \label{figure:LRwaveform_corr_u}
\end{figure}

\subsection{Conditioned Diffusion Process}\label{appendix:cd}
\paragraph{Setting details} For the experiment on conditioned diffusion process, $\mu(z;\phi)$ for all the SIVI variants are parameterized as MLPs with layer widths $[100, 128, 128, 100]$ and ReLU activation functions. 
For SIVI-SM, the $f_\psi$ is a MLPs with layer widths $[100, 512, 512, 100]$ and ReLU  activation functions. 
The initial value of $\phi_\sigma^2$ is $e^{-2}$ for all the SIVI variants. 
We set the learning rate of variational parameters $\phi$ to 0.0002 for KSIVI and 0.0001 for SIVI and SIVI-SM. 
The learning rate of $\psi$ of SIVI-SM is set to 0.0001 and the number of lower-level gradient steps for $\psi$ is 1.
For all methods, we conduct training on variational distributions using 100,000 iterations with a batch size of 128, and the Adam optimizer is employed. 
For SIVI, we use the score-based training as done in \citet{yu2023hierarchical} for a fair time comparison.

\paragraph{Additional results} 
Table \ref{tab:cd_wd_dim} shows the estimated sliced Wasserstein distances from the target distributions to different variational approximations with increasing dimensions d = 50, 100, and 200 on conditional diffusion.
Table \ref{tab:run_time_cd_ustat} shows the training time per 10,000 iterations of KSIVI with the vanilla estimator and KSIVI with the u-stat estimator. 
We see that there exists a tradeoff between computation cost and estimation variance between the vanilla and u-stat overall. 
As shown in the equations (\ref{eq:sto_grad}) and (\ref{eq:sto_grad_u}), the vanilla estimator uses two batches of $N$ samples and requires $N^2$ computation complexity for backpropagation, while the u-stat estimator uses one batch of $N$ samples and requires $\frac{N^2}{2}$ computation. 
For convergence, the variance of the vanilla estimator~(\ref{eq:sto_grad}) is slightly smaller than the variance of the u-stat estimator~(\ref{eq:sto_grad_u}).

Figure \ref{figure:cd_traj_u} depicts the well-trained variational posteriors of SIVI, SIVI-SM, and KSIVI with the U-statistic gradient estimator on the conditioned diffusion process problem.
Figure \ref{figure:cd_loss} shows the training loss of SIVI, SIVI-SM, and KSIVI. 
We observe that the training losses of all the methods converge after 100,000 iterations.

\begin{table}[t]
    \caption{\textbf{Estimated sliced Wasserstein distances (Sliced-WD) on conditional diffusion.} The Sliced-WD objectives are estimated using 1000 samples.}
    \label{tab:cd_wd_dim}
    \centering
    \setlength\tabcolsep{5pt}
    \vskip0.5em
    \begin{tabular}{lccc}
    \toprule
    Methods \textbackslash Dimensionality & 50 & 100 & 200 \\
    \midrule
    SIVI	&0.1266	&0.0981 &0.0756\\
    SIVI-SM	&0.0917	&0.0640	&0.0628\\
    UIVI	&0.0314	&0.0426	&0.0582\\
    KSIVI (u-stat)	&\textbf{0.0134} &0.0182	& \textbf{0.0551}\\
    KSIVI (vanilla)	&\textbf{0.0140} &\textbf{0.0115} &0.0493\\
    \bottomrule
    \end{tabular}
    \vspace{-1.5em}
\end{table}

\begin{table}[t]
    \caption{\textbf{Training time (per 10,000 iterations) for the conditioned diffusion process inference task.} For all the methods, the batch size for Monte Carlo estimation is set to $N=128$.}
    \label{tab:run_time_cd_ustat}
    \centering
    \setlength\tabcolsep{5pt}
    \vskip0.5em
    \begin{tabular}{lccc}
    \toprule
    Methods \textbackslash Dimensionality & 50 & 100 & 200 \\
    \midrule
    KSIVI (u-stat)	&39.09 &58.13 & 64.17\\
    KSIVI (vanilla)	&56.67 &90.48 & 107.84\\
    \bottomrule
    \end{tabular}
    \vspace{-1.5em}
\end{table}

\begin{figure}[t]
    \centering
    \includegraphics[width=\linewidth]{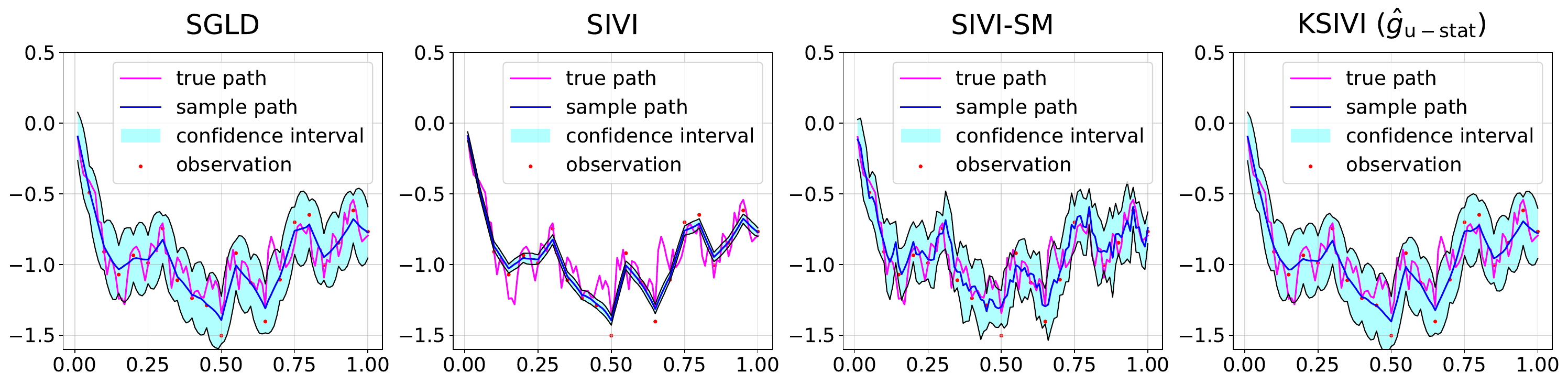}
    \caption{\textbf{Variational approximations of different methods for the discretized conditioned diffusion process.}
    The magenta trajectory represents the ground truth via parallel SGLD. The blue line corresponds to the estimated posterior mean of different methods, and the shaded region denotes the $95\%$ marginal posterior confidence interval at each time step. The sample size is 1000.
    }
    \label{figure:cd_traj_u}
\end{figure}

\begin{figure}[h]
    \centering
    \includegraphics[width=\linewidth]{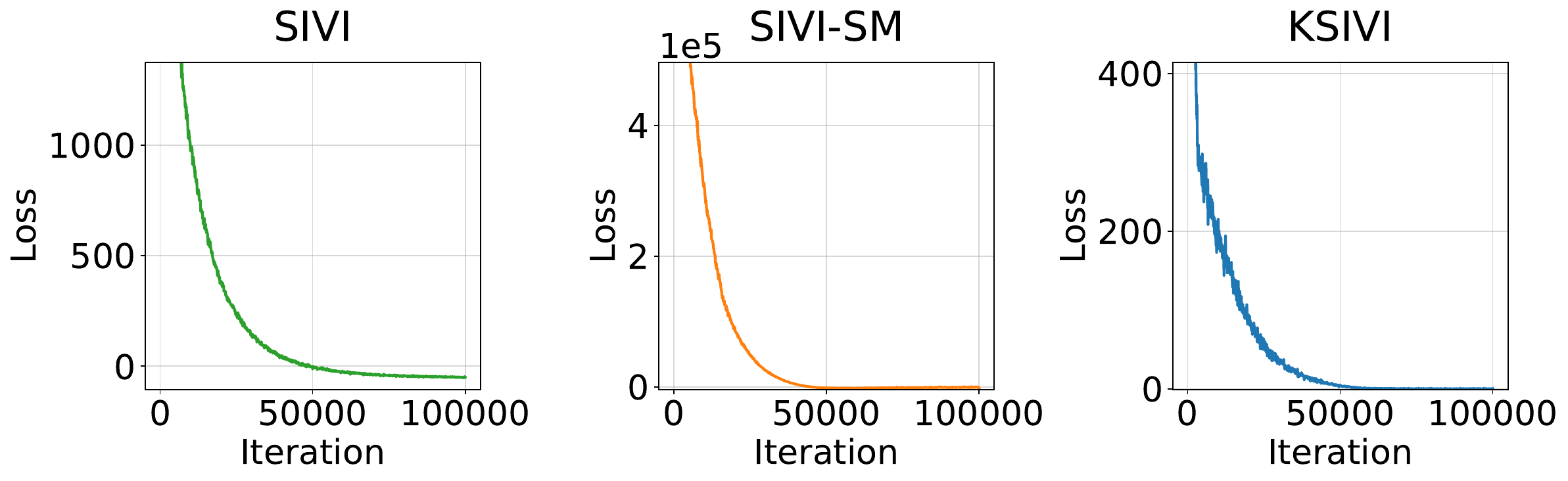}
    \caption{\textbf{Training loss of SIVI in conditioned diffusion process experiment.} The sample size is 1000 for all the methods.}
    \label{figure:cd_loss}
\end{figure}

\subsection{Bayesian Neural Network}\label{appendix:bnn}
\paragraph{Setting details} For the experiments on the Bayesian neural network, we follow the setting in \citet{liu2016stein}. 
For the hyper-parameters of the inverse covariance of the BNN, we select them as the estimated mean of 20 particles generated by SVGD \citep{liu2016stein}.
$\mu(z;\phi)$ for all the SIVI variants are parameterized as MLPs with layer widths $[3, 10, 10, d]$ and ReLU activation functions, where $d$ is the number of parameters in Bayesian neural network.
For SIVI-SM, the $f_\psi$ takes three-layer MLPs with 16 hidden units and ReLU activation functions. 
For KSIVI, we apply a practical regularization fix to thin the kernel Stein discrepancy, as described in \citep{bénard2023kernel}.
For all the SIVI variants, we choose the learning rate of $\phi$ from $\{0.00001, 0.0001, 0.0002, 0.0005, 0.001\}$ and run 20,000 iterations for training.
For SGLD, we choose the step size from the set $\{0.00002, 0.00004, 0.00005, 0.0001\}$, and the number of iterations is set to 10,000 with 100 particles.
\paragraph{Additional results} Table \ref{tab:bnn_rmse_u} shows the additional results of KSIVI with the U-statistic gradient estimator.

\begin{table}[h]
\caption{\textbf{Test RMSE and test NLL of Bayesian neural networks on several UCI datasets}. The results are averaged from 10 independent runs with the standard deviation in the subscripts.
  For each data set, the best result is marked in \textbf{black bold font} and the second best result is marked in \textbf{\textcolor{Sepia!30}{brown bold font}}.
  }  
\label{tab:bnn_rmse_u}
\centering
\vskip0.5em
\setlength\tabcolsep{6.3pt}
\resizebox{\linewidth}{!}{
\begin{tabular}{lcccccccc}
\toprule
 \multirow{2}{*}{Dataset}&\multicolumn{4}{c}{Test RMSE ($\downarrow$)} &\multicolumn{4}{c}{Test NLL ($\downarrow$)}  \\
\cmidrule(l){2-5}\cmidrule(l){6-9}
&  SIVI& SIVI-SM & KSIVI ($\hat{g}_{\textrm{u-stat}}$)& KSIVI ($\hat{g}_{\textrm{vanilla}}$)&SIVI& SIVI-SM & KSIVI ($\hat{g}_{\textrm{u-stat}}$)& KSIVI ($\hat{g}_{\textrm{vanilla}}$) \\
\midrule
\textsc{Boston}       & $\bm{\textcolor{Sepia!30}{2.621}}_{\pm0.02}$ & $2.785_{\pm0.03}$ & $2.857_{\pm0.11}$& $\bm{2.555}_{\pm0.02}$& $\bm{2.481}_{\pm0.00}$ & $2.542_{\pm0.01}$ & $3.094_{\pm0.01}$& $\bm{\textcolor{Sepia!30}{2.506}}_{\pm0.01}$    \\
\textsc{Concrete}     & $6.932_{\pm0.02}$ & $\bm{\textcolor{Sepia!30}{5.973}}_{\pm0.04}$ & $6.861_{\pm0.19}$& $\bm{5.750}_{\pm0.03}$  & $3.337_{\pm0.00}$ & $\bm{3.229}_{\pm0.01}$ & $4.036_{\pm0.01}$& $\bm{\textcolor{Sepia!30}{3.309}}_{\pm0.01}$    \\
\textsc{Power}        & $\bm{3.861}_{\pm0.01}$ & $4.009_{\pm0.00}$ & $3.916_{\pm0.01}$& $\bm{\textcolor{Sepia!30}{3.868}}_{\pm0.01}$   & $\bm{2.791}_{\pm0.00}$ & $2.822_{\pm0.00}$ & $2.944_{\pm0.00}$& $\bm{\textcolor{Sepia!30}{2.797}}_{\pm0.00}$   \\
\textsc{Wine}      & $\bm{\textcolor{Sepia!30}{0.597}}_{\pm0.00}$ & $0.605_{\pm0.00}$ & $\bm{\textcolor{Sepia!30}{0.597}}_{\pm0.00}$& $\bm{0.595}_{\pm0.00}$ & $\bm{\textcolor{Sepia!30}{0.904}}_{\pm0.00}$ & $0.916_{\pm0.00}$ & $\bm{\textcolor{Sepia!30}{0.904}}_{\pm0.00}$& $\bm{0.901}_{\pm0.00}$    \\
\textsc{Yacht}        & $1.505_{\pm0.07}$ & $\bm{0.884}_{\pm0.01}$ & $2.152_{\pm0.09}$& $\bm{\textcolor{Sepia!30}{1.237}}_{\pm0.05}$  & $\bm{\textcolor{Sepia!30}{1.721}}_{\pm0.03}$ & $\bm{1.432}_{\pm0.01}$ & $2.873_{\pm0.03}$& $1.752_{\pm0.03}$    \\
\textsc{Protein}      & $\bm{4.669}_{\pm0.00}$ & $5.087_{\pm0.00}$  & $\bm{\textcolor{Sepia!30}{4.777}}_{\pm0.00}$& $5.027_{\pm0.01}$ & $\bm{2.967}_{\pm0.00}$ & $3.047_{\pm0.00}$ & $\bm{\textcolor{Sepia!30}{2.984}}_{\pm0.00}$& $3.034_{\pm0.00}$    \\
\bottomrule
\end{tabular}
}
\end{table}

\end{document}